\documentclass{article}

\usepackage{microtype}
\usepackage{graphicx}
\usepackage{subcaption}
\usepackage{booktabs}

\usepackage{multirow}
\usepackage{makecell}
\usepackage{placeins}
\usepackage{comment}

\usepackage{hyperref}

\usepackage[preprint]{icml2026}

\usepackage{amsmath}
\usepackage{amssymb}
\usepackage{mathtools}
\usepackage{amsthm}

\usepackage{thmtools,thm-restate}

\usepackage{lipsum}

\newcommand{\N}{\ensuremath{\mathbb{N}}}

\newcommand{\R}{\ensuremath{\mathbb{R}}}

\newcommand{\G}{\ensuremath{\mathcal{G}}}
\newcommand{\E}{\ensuremath{\mathcal{E}}}
\newcommand{\V}{\ensuremath{\mathcal{V}}}

\renewcommand{\a}{\alpha}
\renewcommand{\b}{\beta}

\newcommand{\x}{\ensuremath{\mathbf{x}}}

\renewcommand{\phi}{\ensuremath{\varphi}}

\newcommand{\multiset}[1]{\ensuremath{\{\!\!\{#1\}\!\!\}}}

\newcommand{\mA}{\mathbf{A}}
\newcommand{\mB}{\mathbf{B}}
\newcommand{\mU}{\mathbf{U}}
\newcommand{\mH}{\mathbf{H}}
\newcommand{\mS}{\mathbf{S}}
\newcommand{\mX}{\mathbf{X}}

\newcommand{\rs}{\mathsf{r}}
\newcommand{\mpas}{\mathsf{a}}
\newcommand{\Anorm}{\boldsymbol{\mathsf{A}}}
\newcommand{\mAg}{\mathbf{A}_\mathcal{G}}
\newcommand{\oper}{\boldsymbol{\mathsf{S}}_{\rs,\mpas}}
\newcommand{\up}{\sigma}

\newcommand{\An}{\tilde{\mathbf A}_\G}

\newcommand{\MIX}{\text{MIX}}
\newcommand{\SSM}{\text{SSM}}

\usepackage[capitalize,noabbrev]{cleveref}

\theoremstyle{plain}
\newtheorem{theorem}{Theorem}[section]
\newtheorem{proposition}[theorem]{Proposition}
\newtheorem{lemma}[theorem]{Lemma}

\theoremstyle{definition}
\newtheorem{definition}[theorem]{Definition}

\theoremstyle{remark}

\newcommand{\LapPE}{$\mathcal{O}(k^2 \cdot E)$}
\newcommand{\RWSE}{$\mathcal{O}(k \cdot N^2)$}

\usepackage[disable,textsize=tiny]{todonotes}

\icmltitlerunning{From Message-Passing to Linearized Graph Sequence Models}

\begin{document}

\twocolumn[
  \icmltitle{From Message-Passing to Linearized Graph Sequence Models}

  \icmlsetsymbol{equal}{*}

  \begin{icmlauthorlist}
    \icmlauthor{Joël Mathys}{yyy,equal}
    \icmlauthor{Basil Rohner}{yyy,equal}
    \icmlauthor{Saku Peltonen}{yyy}
    \icmlauthor{Roger Wattenhofer}{yyy}
  \end{icmlauthorlist}

  \icmlaffiliation{yyy}{ETH Zurich, Switzerland}

  \icmlcorrespondingauthor{Joël Mathys}{jmathys@ethz.ch}

  \icmlkeywords{Machine Learning, ICML}

  \vskip 0.3in
]

\printAffiliationsAndNotice{\icmlEqualContribution}

\begin{abstract}
Message-passing based approaches form the default backbone of most learning architectures on graph-structured data.
However, the rapid progress of modern deep learning architectures in other domains, particularly sequence modeling, raises the question of how graph learning can benefit from these advances. 
We introduce Linearized Graph Sequence Models, a framework that recasts message-passing graph computation from the perspective of sequence modeling to simplify architectural choices. 
Our approach systematically separates the computational processing depth from the information propagation depth, allowing core graph architectural decisions to be treated as sequence modeling choices. Specifically, we analyze, both empirically and theoretically, what sequence properties make methods effective for learning and preserving the graph inductive bias.
In particular, we validate our findings, demonstrating improved performance on long-range information tasks in graphs.
Our findings provide a principled way to integrate modern sequence modeling advances into message-passing based graph learning.
Beyond this, our work demonstrates how the separation of processing and information depth can recast central architectural questions as input modeling choices. 
\end{abstract}

\section{Introduction}

Message-passing neural networks (MPNNs) ~\citep{gilmer_neural_2017} have emerged as the dominant architecture for learning on graph-structured data, achieving remarkable success by iteratively aggregating information from local neighborhoods. To incorporate sufficient information from distant nodes in the graph, this sequential aggregation scheme must be repeated many times, resulting in deep networks. Unfortunately, as network depth increases, fixed-sized node embeddings have to keep track of an ever-expanding, exponentially growing receptive field, which ultimately leads to information loss and the collapse of learned representations ~\citep{alon_bottleneck_2020, arnaiz-rodriguez_oversmoothing_2025, di_giovanni_over-squashing_2023}. Moreover, these phenomena are closely related to the difficulty of properly optimizing deep networks due to vanishing gradients ~\citep{arroyo_vanishing_2025}. Meanwhile, modern deep learning has witnessed rapid architectural innovation, particularly in sequence modeling, where Transformers ~\citep{vaswani_attention_2017} and State-Space Models ~\citep{dao_transformers_2024} excel at capturing complex long-range dependencies while maintaining computational efficiency and hardware-friendly designs. This raises the question how graph learning can benefit from architectural advances in modern sequence modeling without losing the inductive bias which makes GNNs effective. 

\begin{figure}[t!]
  \vskip 0.2in
  \begin{center}
    \centerline{\includegraphics[width=\columnwidth]{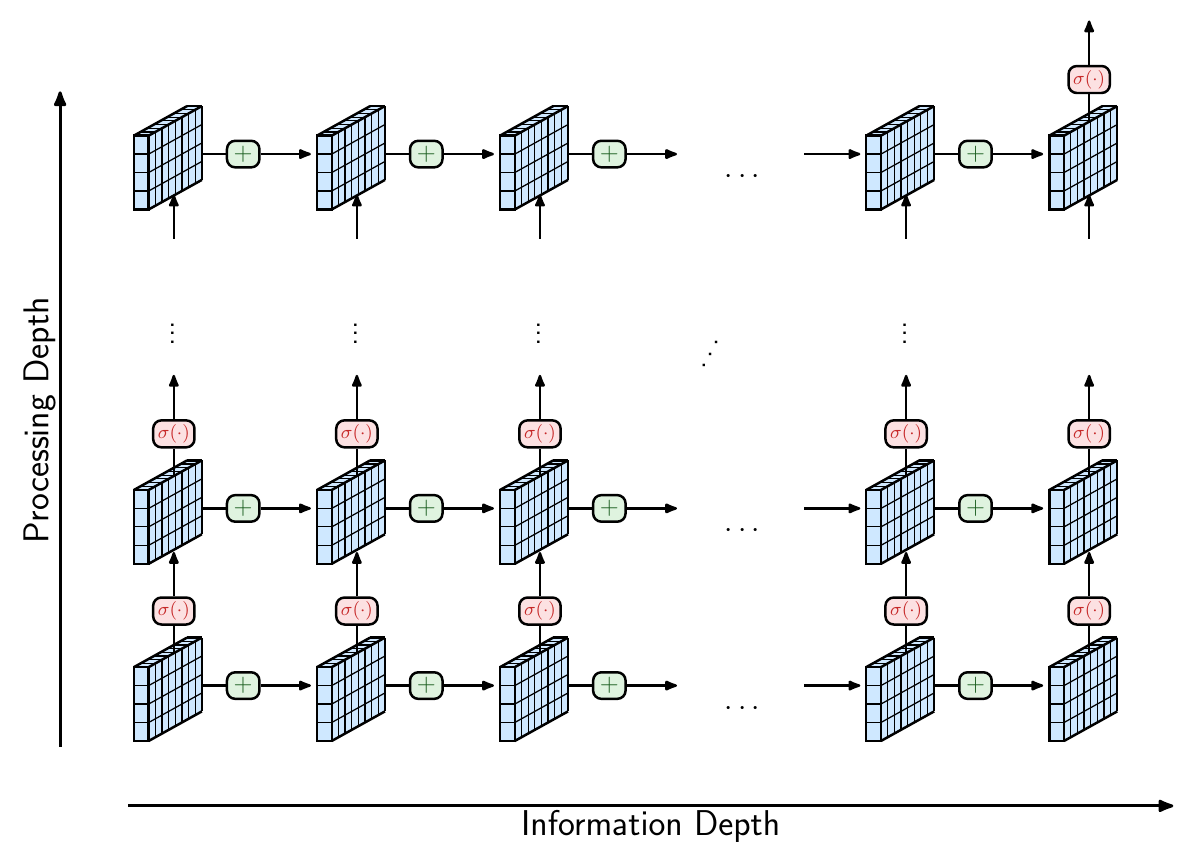}}
    \caption{
      LGSM decouples information propagation (horizontal) from non-linear processing (vertical). Unlike MPNNs that advance both dimensions simultaneously, LGSM enables information to flow across the graph through linearized computation before applying non-linear transformations. This shifts graph architectural focus to the impact of the input sequences.
    }
    \label{icml-historical}
  \end{center}
\end{figure}

Adapting modern sequence modeling techniques for graph learning is not straightforward due to the question on how best to preserve the graph modality and its inductive biases. One line of work attempts to directly model graphs as sequences by casting the graph itself into a sequence based representation. This includes graph transformers ~\citep{ma_graph_2023, rampasek_recipe_2022, stoll_generalizable_2025} that model nodes as tokens, or sequence extraction through random walks ~\citep{tonshoff_walking_2023, behrouz_graph_2024, kim_revisiting_2024, wang_non-convolutional_2024, chen_learning_2024}. Often these approaches struggle to fully preserve graph inductive biases and rely on auxiliary message-passing mechanisms or graph positional encodings to retain necessary structural information ~\citep{grotschla_benchmarking_2026}. An alternative direction focuses on the unrolled computation of message-passing neural networks itself, viewing the iterative updates as a sequential process. These include more sophisticated aggregation schemes ~\citep{ding_recurrent_2024} that leverage modern sequence models, focusing on the recurrent update and their ability to retain information over a prolonged time while maintaining trainability ~\citep{arroyo_vanishing_2025} or casting the architecture directly as a sequence first process \citep{ceni_message-passing_2025, eliasof_graph_2025}. However, particularly the latter requires stacking their derived modules sequentially, reintroducing the very problem of accumulated non-linearities that limits the depth and optimization.

In this work, we introduce Linearized Graph Sequence Models (LGSM), a framework that formulates message-passing based graph learning from the perspective of sequence modeling by explicitly decoupling two dimensions of computation: information depth, how information propagates through the graph, and processing depth, how many non-linear transformations are applied. In standard MPNNs, these two dimensions are tightly intertwined: each sequential layer propagates information one hop further and applies a transformation. LGSM separates these aspects by viewing the unrolled computation of message-passing as a sequence. Crucially, we linearize the computation along the sequence, both ensuring information flow and enabling efficient parallel computation through modern SSM implementations such as Mamba \citep{gu_mamba_2024}. 
This allows core architectural decisions in graph learning to be reframed as sequence modeling choices.

We analyze the computation within LGSM and directly relate the sensitivity of node embeddings \citep{di_giovanni_over-squashing_2023} to the design of the input sequence.
This naturally leads to the question of what makes a sequence extraction mechanism effective for graph learning.
We study both empirically and theoretically how the choice of sequence affects important properties such as stability, sensitivity, or influence, and how they can be improved. Motivated by these insights, we propose a sequence extraction method based on non-backtracking message exchange, which reduces the broadcasting of already known information. We further validate our framework empirically on the ECHO and LRIM benchmark ~\citep{miglior_can_2025, mathys_lrim_2025} on both synthetic graph property prediction tasks as well as real-world molecular datasets, where LGSM achieves strong performance, demonstrating that the separation of information and processing depth is an effective framework for graph learning.
Our main contributions are as follows:

\begin{itemize}
\item We introduce Linearized Graph Sequence Models (LGSM), a framework that decouples information propagation depth from processing depth and allows to recast core message-passing architectural decisions as sequence modeling choices instead.

\item We provide a thorough theoretical analysis connecting the choice of sequence extraction to node sensitivity and information flow, establishing how different extraction mechanisms impact learning on graphs. Further, we propose to use a sequence extraction method based on non-backtracking walks with favorable properties for information propagation.

\item We empirically validate our framework, confirming the effectiveness of our approach on both synthetic graph tasks as well as real-world molecular datasets.
\end{itemize}

\section{Background}

\subsection{Notation}

Let $\G = (\V, \E)$ be an undirected graph with vertex set $\V$
and edge set $\E$, where $|\V| = n$ and $|\E| = m$ denote the number of nodes and edges in the graph. The degree of a node $v$ is $\deg(v)$ and $\mathcal{N}^{\leq k}(v)$ denotes the set of all neighbors of $v$ within distance $k$. Similarly, we denote with $\mathcal{N}^{k}(v)$ the set of nodes for which there exists a path of length exactly $k$. We use $\mathbf A_{\G} \in \R^{n\times n}$ to denote the adjacency matrix defined by $(\mathbf A_{\G})_{u,v} = 1$ if and only if $\{u,v\} \in \E$. The degree matrix $\mathbf D_\G$ is diagonal with $(\mathbf D_\G)_{v,v} = \deg(v)$.
We denote a multiset with $\multiset\cdot$. 
For a matrix $A \in \R^{n \times n}$, $\|A\| = \max_{\|v\|_2=1} \|Av\|_2$ denotes the spectral norm, unless otherwise specified.

\subsection{Graph Neural Networks}

Graph Neural Networks that are based on message-passing~\citep{gilmer_neural_2017} learn node representations 
through iterative neighborhood aggregation. Given initial node features 
$ \mH^{(0)}=\mathbf{X} \in \mathbb{R}^{d \times n}$, each layer $t$ updates the node representations by 
aggregating information from neighbors, while also applying a learnable transformation:
\begin{align*}
    \mathbf{h}_v^{(t+1)} = \sigma\left(\mathbf{h}_v^{(t)}, 
    \phi\left(\multiset{\mathbf{h}^{(t)}_u : u \in \mathcal{N}(v)}\right)\right),
\end{align*}
where $\phi$ is a permutation-invariant aggregation function (e.g., sum, mean) and 
$\sigma$ is a learnable non-linear function, typically an MLP. For many common architectures, such as GCN \citep{kipf_semi-supervised_2017}, the update can be expressed in matrix notation.
\[
    \mathbf H^{(t+1)} = \sigma\left(\mathbf H^{(t)}\mAg \right)
\]
Here, the residual connection to the previous state is modeled as additional self-loops of $\G$.
Crucially, each MPNN layer simultaneously propagates information exactly one hop further in the 
graph \emph{and} applies a non-linear transformation. To incorporate information from 
node features at distance $r$, at least $r$ rounds of message-passing are required, tightly coupling the depth of 
non-linear processing to the information flow within the graph.

\subsection{State-Space Models}

State-Space Models (SSMs) are formulated to capture continuous linear dynamical systems. In order to apply them for sequence models, one considers a discretization of the system over time, essentially mapping the input sequence $\mathbf{x}^{(1)}, \ldots, \mathbf{x}^{(\ell)} \in \mathbb{R}^{d_I}$ 
to an output sequence $\mathbf{y}^{(1)}, \ldots, \mathbf{y}^{(\ell)} \in \mathbb{R}^{d_O}$ 
via the linear recurrence
\begin{align*}
    \mathbf{h}^{(t+1)} &= \mathbf{A}\mathbf{h}^{(t)} + \mathbf{B}\mathbf{x}^{(t+1)}\\
    \mathbf{y}^{(t+1)} &= \mathbf{C}\mathbf{h}^{(t+1)}
\end{align*}
where $\mathbf{A} \in \mathbb{R}^{d_H \times d_H}$, $\mathbf{B} \in \mathbb{R}^{d_H \times d_I}$, 
and $\mathbf{C} \in \mathbb{R}^{d_O \times d_H}$ are the system matrices, and 
$\mathbf{h}^{(t)} \in \mathbb{R}^{d_H}$ is the hidden state.
Interestingly, instead of naively unrolling the linear recurrence, it can be computed much more efficiently on modern hardware. Either by unrolling the recurrence to give an explicit convolution kernel or alternatively by leveraging the associativity of the update with a parallel scan computation ~\citep{gu_mamba_2024}. Moreover, it is possible to switch between the formulation to balance speed and memory requirements during training and inference respectively. 

A key property of modern SSMs is their ability to capture long-range dependencies through appropriate initializations such as the HiPPO framework ~\citep{gu_hippo_2020, gu_efficiently_2021}. These enable the SSM to closely approximate the history of the input sequence using orthogonal polynomials. Moreover, with the introduction of Mamba ~\citep{gu_mamba_2024} that incorporates an input dependent selective update, SSMs achieve performances comparable to the transformer architecture.

\section{Related Work}
There exist different lines of work to include modern architectures for graph learning, however, they do differ quite a bit in the modeling perspective. Random walk based approaches, model the graph modality as a collection of sequences ~\citep{behrouz_graph_2024, chen_learning_2024} and then use the SSM to process the walk sequence. As such, these approaches are orthogonal to our investigation into message-passing centric computation. On the other hand \citet{huang_what_2024} translates the convolution view of the SSM from linear sequences to graph topologies. Closely related are the approaches that include SSMs in the design of the aggregation function \citep{ding_recurrent_2024}, but otherwise do not consider the full architecture. Some works consider SSMs to process the output of  sequential non-linear MPNN \citep{eliasof_graph_2025, arroyo_vanishing_2025} or even linearized variants \citep{ceni_message-passing_2025}. However, they do not consider the separation of processing and information gathering and stack the blocks sequentially. This reintroduces non-linearities and prevents fully parallelizable processing and complete theoretical analysis along the full information depth. We refer to Appendix \ref{app:related work} for an in depth discussion of related works.

\section{Method}

We propose the Linearized Graph Sequence Model (LGSM) framework, to reformulate message passing based graph learning with a focus on sequence modeling. This perspective explicitly decouples two fundamental dimensions: \textit{processing depth}, 
how much compute we spend on transforming information, and \textit{information depth}, how information propagates throughout the graph. Before presenting our model in detail, we provide a motivating example showing how graph computation can be viewed as unrolled sequence computation. 

\begin{figure*}[ht]
  \centering
  \vskip 0.2in
  \makebox[0.95\textwidth][c]{
  \begin{subfigure}[t]{0.30\textwidth}
    \centering
    \includegraphics[width=\linewidth]{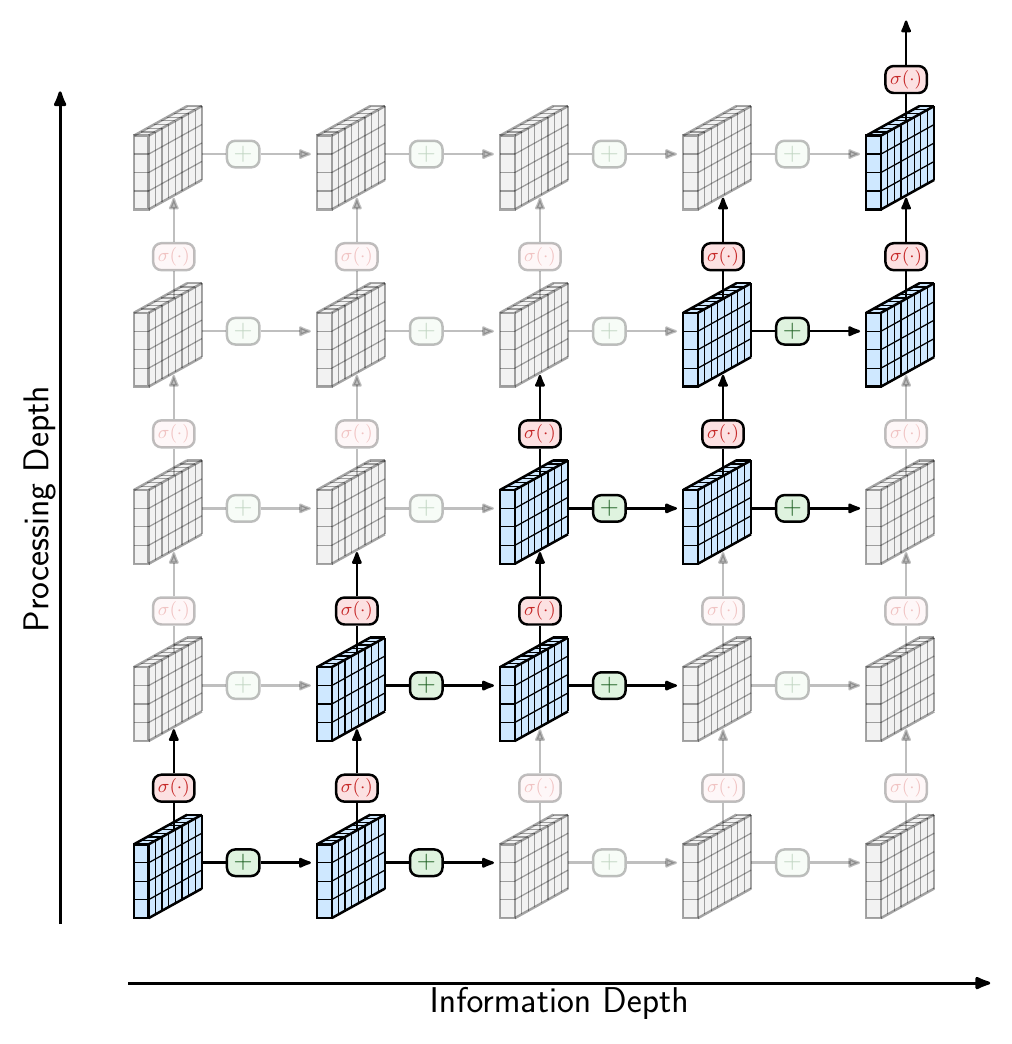}
    \caption{MPNN}
    \label{fig:storyline-1}
  \end{subfigure}
  \hfill
  \begin{subfigure}[t]{0.30\textwidth}
    \centering
    \includegraphics[width=\linewidth]{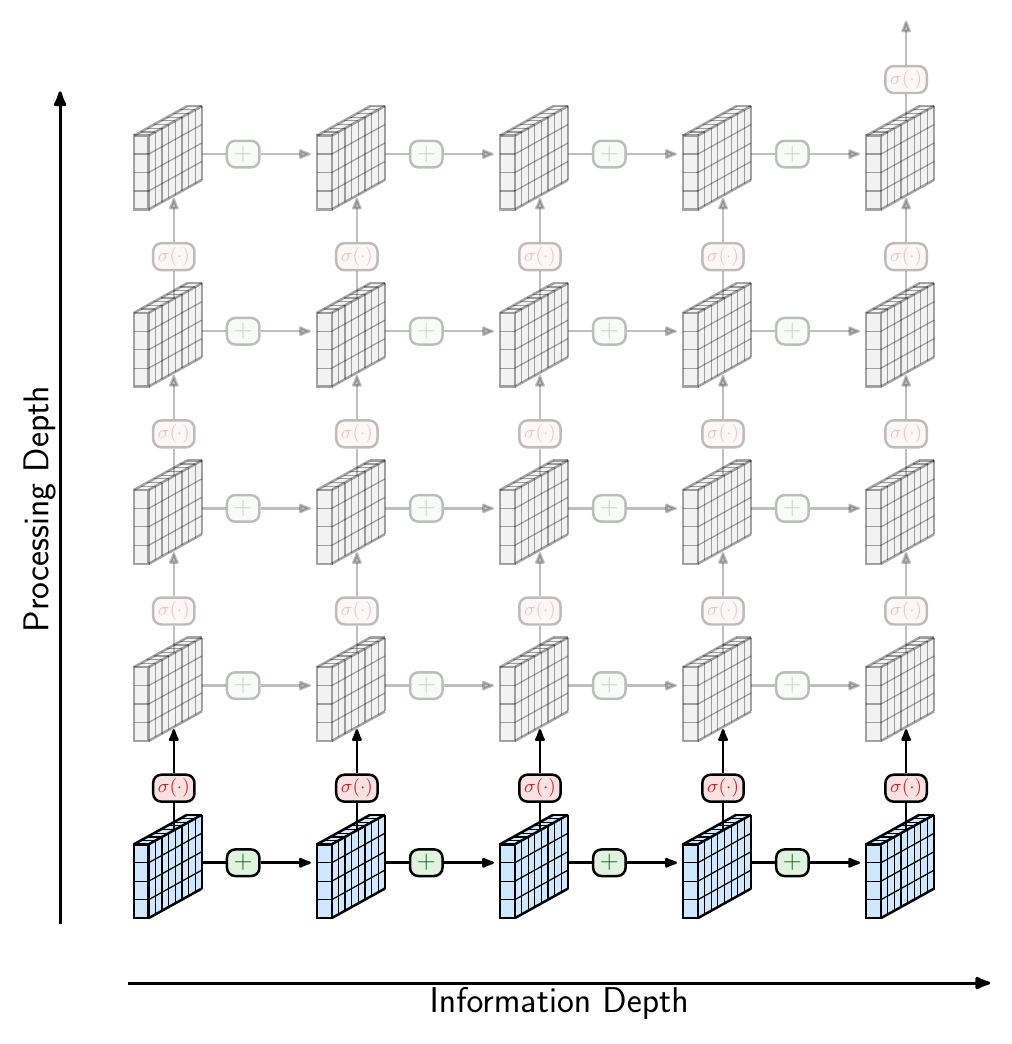}
    \caption{Linear MPNN}
    \label{fig:storyline-2}
  \end{subfigure}
  \hfill
  \begin{subfigure}[t]{0.30\textwidth}
    \centering
    \includegraphics[width=\linewidth]{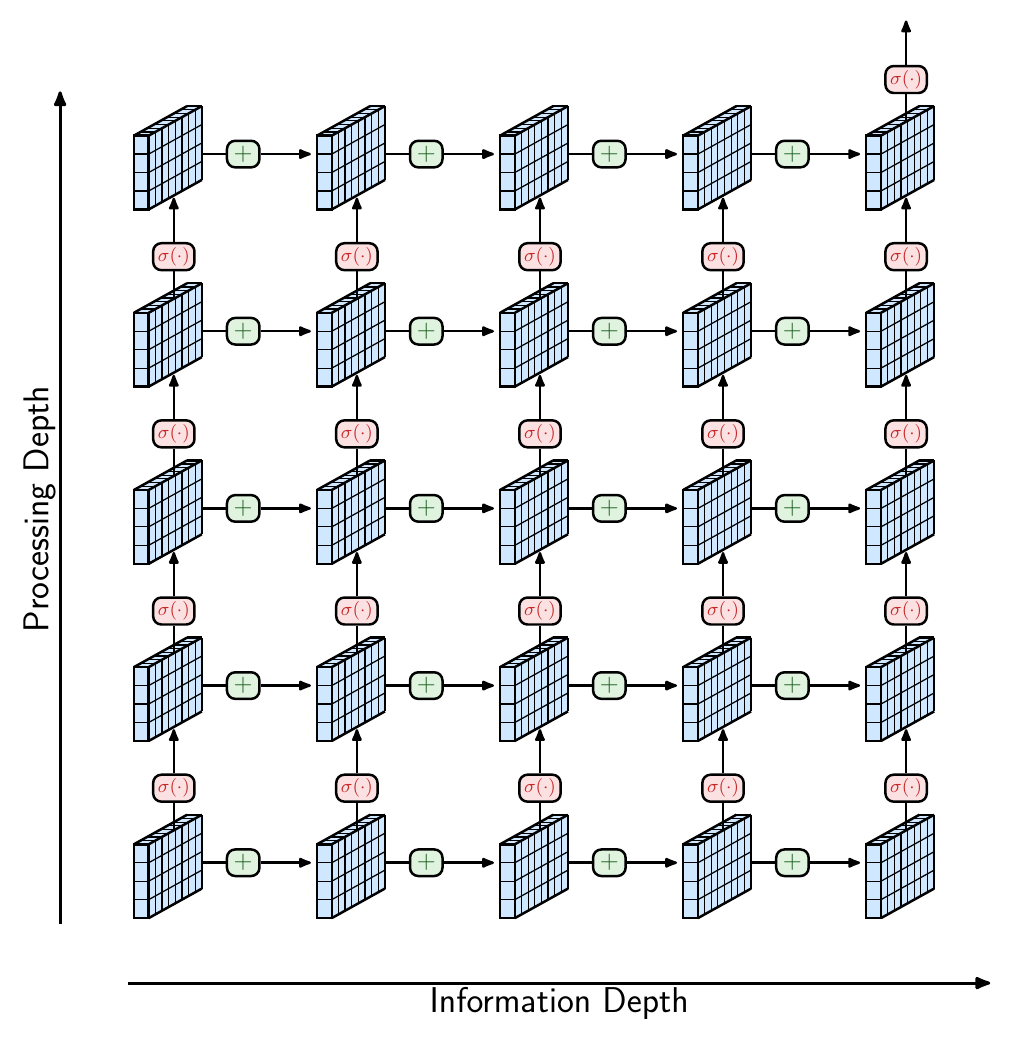}
    \caption{LGSM (our method)}
    \label{fig:storyline-3}
  \end{subfigure}
    }
  \caption{
Visualization of computation in graph neural networks with respect to information depth and processing depth. Standard MPNNs couple both dimensions, propagating and updating information through non-linear transformations at the same time, advancing diagonally. Linear MPNNs propagate information without any non-linearities, but have limited transformation capabilities. Our method, stacks the linearized rows to combine both well-conditioned linear information flow with non-linear processing capabilities.
}
  \label{fig:storyline}
\end{figure*}

The key observation is that stacking $L$ MPNN layers can be viewed as a recurrent graph update. Consider a GCN-based update rule with skip connection from the input features $\mathbf X$, with initial $\mathbf H^{(0)} = \mathbf 0$. Each layer propagates information one hop and then applies a non-linearity. As such, the computational structure resembles that of a recurrent neural network, where the sequence corresponds to successive rounds of neighborhood aggregation. However, unlike standard RNNs that receive an external input at each step, the GCN derives an input implicitly. As we show next, linearizing this recurrence results in a sequence structure.

\begin{align*}
\mathbf{H}^{(t+1)}
&= \sigma\!\Big(
\underbrace{\mH^{(t)}}_{\substack{\text{previous}\\\text{hidden state}}},
\underbrace{\mU^{(t)}}_{\substack{\text{next}\\\text{input}}}
\Big) &\text{RNN formulation}\\
\mH^{(t+1)} &= \mX + \sigma(\mH^{(t)}\mAg) &\text{GCN formulation}
\end{align*}

 Where $\mU$ corresponds to the input sequence of the RNN. We visualize this unrolled computation as a grid in Figure \ref{fig:storyline-1}, where the horizontal axis represents information depth and the vertical axis represents processing depth. The MPNN advances both dimensions simultaneously at each step. This leads to a diagonal, staircase-like computational path. Unfortunately, this exhibits two inherent drawbacks. First, the computation has to be done in sequential order, which is hard to parallelize on modern hardware. Second, due to the repeated stacking of non-linearities, it can be hard to properly optimize these networks due to instability or vanishing gradients. Following modern state-space approaches \citep{gu_efficiently_2021}, we address these concerns by linearizing the update along the information axis. Removing the non-linearity, akin to \citet{ceni_message-passing_2025}, allows information to propagate without accumulating non-linear transformations. 
 
\begin{align*}
\mathbf{H}^{(t+1)}
&= \mA \mH^{(t)}+ \mB \mU^{(t)} &\text{linearized RNN} \\
&= \sum_{k=0}^{t}\mA^{k}\mB\mU^{(t-k)}\\
\mH^{(t+1)} &= \mB\mX + \mA\mH^{(t)}\mAg &\text{linearized GCN}\\
&= \sum_{k=0}^{t}\mA^k\mB \underbrace{\mX\mAg^k}_{\mU^{(t-k)}}
\end{align*}
Where $\mA,\mB \in \mathbb{R}^{d\times d}$ are learnable matrices. Through this lens, the linearized MPNN  is equivalent to applying an SSM to a graph-based input sequence, where each sequence element contains node features propagated through the graph: $\mU_{L:0} = \mathbf X,\mathbf X\mAg, \mathbf X \mAg^{2}, \ldots, \mathbf X\mAg^{L}$. In the grid of Figure \ref{fig:storyline-2}, this corresponds to only traversing the horizontal axis, information flows through the graph without any non-linear transformations.
We use this perspective to build our framework. Similarly to the architectural design of modern SSMs, we stack multiple such linearized horizontal blocks on top of each other and connect them with non-linear transformations along the processing depth. The resulting model architecture depicted in Figure \ref{fig:storyline-3} now precisely separates information and processing depth. 

\subsection{Linearized Graph Sequence Model}

We now describe the LGSM architecture in detail. Given an input graph $\G = (\V, \E)$ with node features $\mX \in \mathbb{R}^{d \times n}$, an LGSM transforms these features through four components: sequence extraction, SSM processing, a feed-forward network, and graph mixing. In favor of clarity we omit some of the common optimization techniques such as skip connections and normalizations in the mathematical definition. A conceptual overview is shown in Figure \ref{fig:architecture}.

\textit{Sequence Extraction}: The Graph Sequence Encoder gets the input graph with its node features and converts them into a sequence of $L$ node embeddings $\smash{\mS^{(i)}_{\text{in}}\in \mathbb{R}^{d\times n}}$. 
\begin{align*}
    \mS^{(0)}_{\text{in}}, \mS^{(1)}_{\text{in}}, \ldots, \mS^{(L-1)}_{\text{in}} = \textsc{Seq}(\mX, \mAg)
\end{align*}
 In the next section, we discuss how to extract an appropriate sequence in more detail. A simple choice could be powers of the adjacency matrix: $\smash{\mS^{(i)}_{\text{in}} = \mX \mAg^{i}}$.

\textit{SSM Layer}: The sequence is then processed by a state-space model, which operates independently on each node's sequence. We use Mamba2 ~\citep{dao_transformers_2024} as the SSM module. The node dimension is viewed as a batch dimension and the SSM processes $L$ sequence elements of dimension~$d$.
\begin{align*}
    \mS^{(0)}_{\text{SSM}}, \mS^{(1)}_{\text{SSM}}, \ldots, \mS^{(L-1)}_{\text{SSM}}  = \textsc{SSM}(\mS^{(0)}_{\text{in}}, \mS^{(1)}_{\text{in}}, \ldots, \mS^{(\ell)}_{\text{in}})
\end{align*}
This enables information to flow along the sequence dimension efficiently via parallel scan computation.
Afterwards, each node embedding is fed through a feed-forward network $\sigma_1$ for non-linear processing.
\begin{align*}
    \mS^{(i)}_{\text{FFN}} = \sigma_1(\mS^{(i)}_{\text{SSM}})
\end{align*}

\textit{Graph Mixing}: Finally, to have information mixing between different nodes, we update the sequence elements
with the graph-propagated features of the preceding element before we apply another FFN and feed the output to the next block. Note that because the previous state is mixed between nodes this does not further increase the information depth.
\begin{align*}
    \mS^{(0)}_{\text{MIX}} = {\mS^{(0)}_{\text{FFN}}} \qquad \mS^{(i)}_{\text{MIX}} = \sigma_2\big({\mS^{(i)}_{\text{FFN}}} + {\mS^{(i-1)}_{\text{FFN}}} \mAg\big)
\end{align*}

A complete LGSM model stacks $D$ such blocks, where the output of one block becomes the input to the next. The final node representations are obtained from the last sequence element $\smash{\mS_{\text{out}}^{(\ell)}}$ and passed to a graph decoder to predict node or graph level embeddings.

\begin{figure}[ht!]
  \vskip 0.2in
  \begin{center}
    \centerline{\includegraphics[width=0.9\columnwidth]{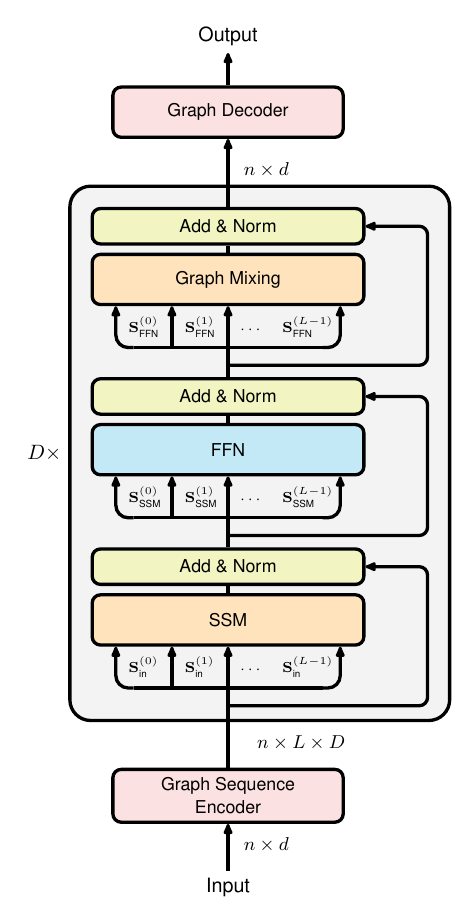}}
\caption{Overview of the LGSM architecture. The Graph Sequence Encoder converts graph input into a sequence. Then, each block applies an SSM for each node along the sequence dimension with additional  non-linear transformations and graph mixing. The number of stacked blocks controls the processing depth while the length of the sequence controls information depth.}
    \label{fig:architecture}
  \end{center}
\end{figure}

The linearization along the information depth provides flexibility in how computation is performed. The only location, where the sequence elements all depend on each other is within the SSM module. Note, that efficient SSM implementations give the option to compute row-wise, using efficient and fast parallel scans processing the entire sequence length at once. Alternatively, the computation can be unrolled column-wise, processing all nodes for a given sequence position before advancing to the next. This allows for flexibility and balancing memory and speed depending on hardware constraints and sequence length.

\section{Theoretical Evaluation}

In this section, we analyze the LGSM architecture to understand how its computation makes use of the information depth. For traditional MPNNs, prior work by ~\citet{di_giovanni_over-squashing_2023} gives an upper bound for local node sensitivity:

\begin{align*}
    \left\| \frac{\partial \mathbf{h}_{v}^{(m)}}{\partial \mathbf{h}_{u}^{(0)}}\right\| \leq  (\underbrace{\vphantom{\oper^m}c_{\up}wp}_{\mathrm{model}})^{m}\underbrace{(\oper^{m})_{vu}}_{\mathrm{topology}},
\end{align*}
 where $\oper := c_{\rs}\mathbf{I} + c_{\mpas}\Anorm\in\R^{n\times n}$ is the graph shift operator and $c_\sigma, w, p$ are parameters of the models non-linearity. For nodes at least distance $d$ apart, this bound implies that sensitivity decays with $d$.

We now characterize sensitivity for LGSM. Consider a single LGSM block 
with sequence length $L$. We quantify the dependence of the final output embedding $\mathbf y_v$ on the input feature $\mathbf x_w$.

\begin{restatable}[Local Sensitivity for $1$-Block LGSM]{theorem}
{sensOneLayer}
    \label{thm:sens_1gsm}
        Consider a $1$-block LGSM of sequence length $L$ with an SSM defined by system matrices $(\mathbf A, \mathbf B, \mathbf C)$. Denote the input and output features by $\mathbf X, \mathbf S_{\text{out}}^{(L)} \in \R^{n\times d}$, respectively, where $\mathbf x_v = (\mathbf X)_{v,:}$, is the feature vector of vertex $v \in \mathcal V$.
        Suppose that the FNNs $\sigma_\SSM, \sigma_\MIX$ are $\mu_\SSM$- and $\mu_\MIX$-regular, respectively, 
        and that $\mathbf A$ is normal with spectral radius $\rho(\mathbf A) \approx 1$. Then, for any vertex $v \in \mathcal V$:
    \begin{align*}
        \left\|\frac{\partial \mathbf S_{\text{out}}^{(L)}}{\partial \mathbf x_{v}}\right\|
        \le \gamma\sum_{k=0}^L\left\|\frac{\partial \mathbf S_{\text{in}}^{(k)}}{\partial \x_v}\right\| + \gamma\|\mathbf A_\G\|\sum_{k=0}^{L-1}\left\|\frac{\partial \mathbf S_{\text{in}}^{(k)}}{\partial \x_v}\right\|
    \end{align*}

    where $\gamma$ is a constant in $\|\mathbf B\|, \|\mathbf C\|, \mu_{\textup{SSM}}$, and $\mu_{\textup{MIX}}$.
\end{restatable}
The bound reveals two complementary mechanisms which are at play. The first term (weighted by $\mAg$) captures sensitivity through graph mixing, which propagates information between nodes via the adjacency structure. The second term captures sensitivity through the SSM dynamics, which processes the sequence regardless of graph topology.
Therefore, the SSM ensures that there is a direct sensitivity contribution between the nodes that depends on the initial extracted sequence. This ensures sensitivity is maintained as long as there is at least one sequence element $j$ which is sufficiently dependent on the pair $v,w \in \V$ so that $\left\|\frac{\partial\mathbf s^{(j)}_v}{\partial\mathbf x_w}\right\|$ is large.

We extend our analysis to the full LGSM model with stacked blocks. For clarity, we present the bound without graph mixing here. For the complete bound and  proof we refer to the Appendix. Note that the graph mixing component only contributes additional gradient flows, thus
improving the bound in \Cref{thm:sens_noGM}.

\begin{restatable}[Local Sensitivity for LGSM]{theorem}{sensNoMix}
    \label{thm:sens_noGM}
    Consider a $D$-block LGSM without graph mixing layers and of sequence length $L$,
    with an SSM defined by system matrices $(\mathbf A, \mathbf B, \mathbf C)$.
    Denote the input and output features by $\mathbf X, \mathbf S_{\text{out}}^{(L, D)} \in \R^{n\times d}$, respectively, where $\mathbf x_v = (\mathbf X)_{v,:}$, is the feature 
    vector of vertex $v \in \mathcal V$. Suppose that the FFN $\sigma_\SSM$ is $\mu_\SSM$-regular and that 
    $\mathbf A$ is normal with spectral radius $\rho(\mathbf A) \approx 1$. Then, for any vertex $v \in \mathcal V$:
\begin{align*}
    \left\|\frac{\partial \mathbf S_{\text{out}}^{(L, D)}}{\partial \mathbf x_v}\right\| \le \gamma^D\sum_{k=0}^L \binom{L-k+D-1}{D-1} \left\|\frac{\partial \mathbf S_{\text{in}}^{(k,1)}}{\partial \mathbf x_v}\right\|
\end{align*}
    where $\gamma$ is a constant in $\|\mathbf B\|, \|\mathbf C\|$, and $\mu_{\textup{SSM}}$.
\end{restatable}

The bound separates into two components consisting of the architectural term $\delta^d$ and 
the sequence dependent sensitivities $\left\|\frac{\partial \mathbf s_v^{(k)}}{\partial \mathbf x_w}\right\|$. The binomial 
coefficients count the number of distinct computational paths from input to output across the processing and information depth dimension depicted in Figure~\ref{fig:storyline-3}.

 This contrasts with standard MPNNs, where the sensitivity is bounded by the powers of the graph operator. Although LGSM also incorporates graph-dependent propagation through graph mixing layers, the SSM component ensures that sensitivity is not only determined by it. Instead, the bound depends on all input sequence elements, relying on the fact that at least one element captures the node pair relationship. Therefore, the LGSM framework allows us to shift key graph learning questions toward the appropriate modeling of the sequence extraction.

\section{Sequence Extraction Mechanisms}

Our theoretical analysis reveals that the sensitivity of the final node embeddings has a direct dependence on all elements of the extracted sequence. Specifically, for a pair of nodes $v,w \in \V$ the sensitivity $\left\|\frac{\partial \mathbf y_v}{\partial \mathbf x_w}\right\|$ depends on a weighted sum over $\left\|\frac{\partial \mathbf s_v^{(k)}}{\partial \mathbf x_w}\right\|$. Crucially, it suffices for at least one sequence element to capture strong interaction between the nodes to effectively capture dependencies between them. This motivates a more systematic analysis of what properties make a sequence extraction mechanism effective for graph learning. In the following, we identify five desirable properties for effective sequences: 
\begin{itemize}
    \item \textbf{Efficiency} The sequence should be computable through preprocessing or have an efficient non-sequential formulation that leverages modern hardware. This is to avoid sequence extraction becoming a bottleneck for the overall architecture.
    \item \textbf{Stability} Sequence elements should be numerically bounded to avoid destabilizing the computation, particularly as sequence length increases.
    \item \textbf{Informativeness} Each sequence element should contain meaningful information and be sufficiently distinct from other elements. 
    \item \textbf{Sensitivity} For any pair of nodes $v,w \in \V$ the node sensitivity $\left\|\frac{\partial \mathbf s_v^{(k)}}{\partial \mathbf x_w}\right\|$ should be greater than zero for at least one sequence element $k$. 
    \item \textbf{Relative Influence} The signal from relevant node pairs should be sufficiently strong considering other nodes in the graph. More specifically consider
    \begin{align*}
    I_{v,k}(w) = \mathbf e^T \left[\frac{\partial \mathbf s^{(k)}_v}{\partial \mathbf x_w}\right] \mathbf e \Big / \left(\sum_{u \in \V} \mathbf e^T \left[\frac{\partial \mathbf s^{(k)}_v}{\partial \mathbf x_u}\right] \mathbf e\right)
    \end{align*}
    as the relative influence of $w$ on $v$ on the $k$th sequence element, following the influence distribution of ~\citep{xu_representation_2018}. Here $\mathbf e$ is the all ones vector. 
\end{itemize}

The most natural choice for sequence extraction is perhaps using the powers of the adjacency matrix $\mS^{(i)}_A = \mAg^i \mX$ or its normalized version $\smash{\mS^{(i)}_{\hat{A}} = (\mathbf D_\G^{-\frac{1}{2}} \mathbf A_\G \mathbf D_\G^{-\frac{1}{2}})^i \mX}$. These correspond closely to the message passing scheme incorporated in standard MPNN architectures, where the $t$-th iteration aggregates information along all walks of length $t$.

Considering the properties above, both sequences can in principle be computed efficiently through direct powers or preprocessing of the adjacency powers. Recall, that in both variants the $t$-th element of the sequence corresponds to the aggregation along all walks of length $t$. This can quickly lead to instabilities in $\mS^{(i)}_A$ as the number of walks increases exponentially, but can be addressed with proper degree normalization. If a pair of nodes is connected, eventually there will be a sequence element which ensures non-zero sensitivity. Note that this interaction might still be small as it is proportional to the ratio of paths of length $t$ between $u$ and $v$ over all possible paths of length $t$ originating from them.  

While the adjacency powers satisfy efficiency, stability, and sensitivity, let us consider the following scenario for the last two criteria. The adjacency powers always considers \textit{all} paths of length $t$. This includes paths that immediately "backtrack", immediately traverse an edge in the other direction again, and therefore contain a short two cycle. As a consequence, if a node $v$ can be reached at time $t_0$ from $v$, then for all $t'>0$ the entry will always be non-zero for $(\mAg^{t_0+2t'})_{uv}$. Each following sequence element will contain not only new information about nodes reachable at distance $t$ but also redundant information about nodes reachable at shorter distances through backtracking paths. This redundancy undermines the informativeness of the sequence and also its relative influence, as the set of nodes that have non-zero sensitivity becomes increasingly larger. 

To address this limitation of the adjacency powers, we propose to use an alternative sequence extraction mechanism based on non-backtracking walks. Instead of considering all paths of length $k$ at step $k$, we only consider paths of length $k$ which at no point immediately reverse direction. Thus, eliminating the two-cycle amplification that causes excessive redundant information. We refer to the non-backtracking (NBT) matrix as $\mB_{\G}$, where $(\mB_{\G})_{uv}$ contains the number of non-backtracking walks from $u$ to $v$. We remark that ~\citet{park_non-backtracking_2024} have also considered non-backtracking message-passing and its sensitivity, although solely by replacing them in the standard MPNN formulation. 

\begin{definition}[Non-Backtracking Matrix]
    \label{def:NBT}
    Let $\G = (\V, \E)$ be an undirected simple graph. The non-backtracking matrix $\mathbf B_\G \in \R^{n\times n}$ is defined by the recurrence
    \begin{align*}
        \mathbf B_\G^{(0)} &= \mathbf I\\ 
        \mathbf B_\G^{(1)} &= \mathbf A_\G\\ 
        \mathbf B_\G^{(2)} &= \mathbf A_\G^2 - \mathbf D_\G\\
        \mathbf B_\G^{(t+2)} &= \mathbf A_\G \mathbf B_\G^{(t+1)} - (\mathbf D_\G - \mathbf I)\mathbf B_\G^{(t)}
    \end{align*}
\end{definition}

While the computation is sequential, we can still precompute the sequence of $\mathbf{I}, \mathbf B_\G, \mathbf B_\G^2, ... \mathbf B_\G^{l}$ if necessary or directly multiply with the feature matrix to avoid large multiplications between adjacency matrices. The key benefit of non-backtracking walks over the regular adjacency powers lies in their improved relative influence for distant nodes by reducing the amount of redundant backtracking information.

\begin{restatable}[Relative Influence of NBT]{theorem}{influence}
    \label{thm:influence}
        Let $c<1$ and let $I^A_{v,k}(w)$ (resp. $I_{v,k}^B(w)$) denote the influence of $w$ on the $k$th sequence element of $v$ captured using the adjacency sequence $\mA$ (resp. non-backtracking $\mB$). There exists an infinite family of graphs $\mathcal{G}$ where $I_{v,k}^A(w)$ decays exponentially faster with distance $k$ compared to $I_{v,k}^B(w)$:
\begin{align*}
    \frac{I^A_{v,k}(w)}{I^B_{v,k}(w)}\leq c^{k-1},
\end{align*}
\end{restatable}
See Appendix~\ref{sec:seq_extraction} for the proof.

We also validate these insights empirically by testing the different sequence methods on the ECHO eccentricity task. In Figure \ref{fig:eccentricity}, we train LGSM models with different sequence lengths and the considered sequence types. While the performance increases for all sequence types the longer the sequences are, using naive adjacency powers leads to numerical instabilities for longer sequences. Moreover, the $\mB_\G$ exhibit better performance compared to the adjacency matrices across all considered sequence lengths.

\begin{figure}[ht!]
  \vskip 0.2in
  \begin{center}
    \centerline{\includegraphics[width=\columnwidth]{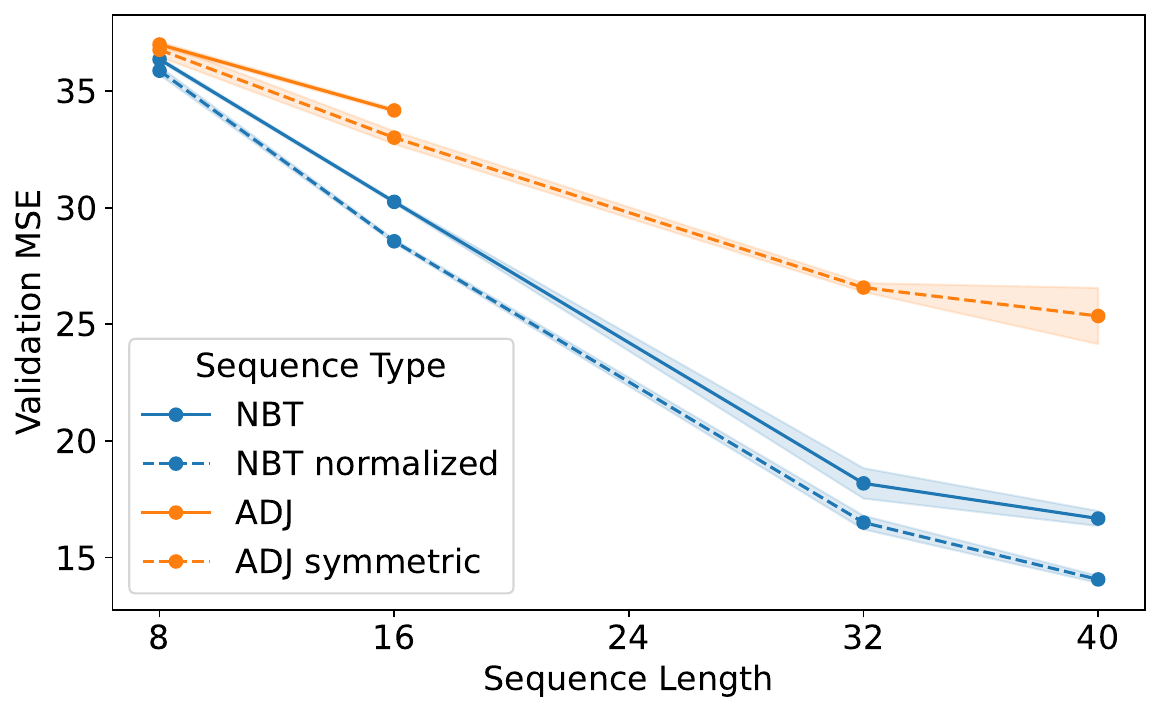}}
    \caption{
      Comparison of sequence extraction mechanisms on the ECHO eccentricity task with graphs of diameter up to 40. We evaluate non-backtracking (NBT) and adjacency powers, both original and normalized. Each datapoint represent the mean of three trained models of a specific sequence length using four LGSM blocks. Therefore the number of trainable parameters is the same for all sequence lengths. We observe that the performance consistently improves for longer sequences. Moreover, non-backtracking sequences are more informative and lead to consistently better performance across all considered sequence lengths. The unnormalized adjacency sequence causes numerical instabilities, leading to NaN values in computation.
    }
    \label{fig:eccentricity}
  \end{center}
\end{figure}

{\color{orange}

}

\section{Empirical Evaluation}

\begin{table*}[tb]
  \caption{Performance on synthetic graph property prediction tasks from ECHO-Synth, which are designed to be challenging for long-range information exchange. Results show MSE and MAE for diameter (DIAM), eccentricity (ECC), and single-source shortest path (SSSP) prediction. LGSM demonstrates strong performance across all datasets, particularly on the eccentricity prediction.}
  \label{tab:test_results_all}
  \begin{center}
    \begin{small}
      \begin{sc}
        \begin{tabular}{lrrrrrr}
\toprule
\multicolumn{1}{c}{Dataset} & \multicolumn{2}{c}{diam} & \multicolumn{2}{c}{ecc} & \multicolumn{2}{c}{sssp} \\
\cmidrule(lr){2-3}
\cmidrule(lr){4-5}
\cmidrule(lr){6-7}
 & MSE $\downarrow$ & MAE $\downarrow$ & MSE $\downarrow$ & MAE $\downarrow$ & MSE $\downarrow$ & MAE $\downarrow$ \\
\midrule
A-DGN & 4.818 \tiny{$\pm$ 0.108} & 1.151 \tiny{$\pm$ 0.038} & 35.967 \tiny{$\pm$ 0.492} & 4.981 \tiny{$\pm$ 0.037} & 4.425 \tiny{$\pm$ 0.879} & 1.176 \tiny{$\pm$ 0.140} \\
DRew & 3.756 \tiny{$\pm$ 0.170} & 1.243 \tiny{$\pm$ 0.047} & 32.247 \tiny{$\pm$ 0.148} & 4.651 \tiny{$\pm$ 0.020} & 6.589 \tiny{$\pm$ 0.015} & 1.279 \tiny{$\pm$ 0.011} \\
GCN & 22.872 \tiny{$\pm$ 2.766} & 3.832 \tiny{$\pm$ 0.262} & 39.706 \tiny{$\pm$ 0.460} & 5.233 \tiny{$\pm$ 0.034} & 9.743 \tiny{$\pm$ 0.757} & 2.102 \tiny{$\pm$ 0.094} \\
GCNII & 9.696 \tiny{$\pm$ 0.568} & 2.005 \tiny{$\pm$ 0.093} & 39.911 \tiny{$\pm$ 0.518} & 5.241 \tiny{$\pm$ 0.030} & 10.369 \tiny{$\pm$ 3.575} & 2.128 \tiny{$\pm$ 0.429} \\
GIN & 7.238 \tiny{$\pm$ 1.153} & 1.630 \tiny{$\pm$ 0.161} & 34.454 \tiny{$\pm$ 1.201} & 4.869 \tiny{$\pm$ 0.092} & 11.868 \tiny{$\pm$ 2.689} & 2.234 \tiny{$\pm$ 0.271} \\
GPS & 10.454 \tiny{$\pm$ 0.610} & 2.160 \tiny{$\pm$ 0.098} & 33.346 \tiny{$\pm$ 0.226} & 4.758 \tiny{$\pm$ 0.021} & 1.255 \tiny{$\pm$ 0.113} & 0.472 \tiny{$\pm$ 0.050} \\
GRIT & 3.877 \tiny{$\pm$ 0.295} & 1.014 \tiny{$\pm$ 0.046} & 38.667 \tiny{$\pm$ 1.903} & 5.091 \tiny{$\pm$ 0.158} & 0.147 \tiny{$\pm$ 0.083} & 0.121 \tiny{$\pm$ 0.013} \\
GraphCON & 16.427 \tiny{$\pm$ 1.419} & 2.969 \tiny{$\pm$ 0.189} & 43.505 \tiny{$\pm$ 0.017} & 5.474 \tiny{$\pm$ 0.001} & 52.104 \tiny{$\pm$ 0.016} & 5.734 \tiny{$\pm$ 0.011} \\
PH-DGN & 6.699 \tiny{$\pm$ 2.728} & 1.627 \tiny{$\pm$ 0.398} & 37.510 \tiny{$\pm$ 1.416} & 5.068 \tiny{$\pm$ 0.126} & 4.656 \tiny{$\pm$ 3.013} & 1.323 \tiny{$\pm$ 0.485} \\
SWAN & 4.950 \tiny{$\pm$ 0.265} & 1.121 \tiny{$\pm$ 0.070} & 34.208 \tiny{$\pm$ 0.578} & 4.840 \tiny{$\pm$ 0.045} & 2.905 \tiny{$\pm$ 1.556} & 0.896 \tiny{$\pm$ 0.232} \\
\midrule
LGSM (ours) & \textbf{3.089 \tiny{$\pm$ 0.389}} & \textbf{0.859 \tiny{$\pm$ 0.044}} & \textbf{13.549 \tiny{$\pm$ 0.539}} & \textbf{2.430 \tiny{$\pm$ 0.070}} & \textbf{0.040 \tiny{$\pm$ 0.008}} & \textbf{0.076 \tiny{$\pm$ 0.009}} \\
\bottomrule
\end{tabular}
      \end{sc}
    \end{small}
  \end{center}
  \vskip -0.1in
\end{table*}

\begin{table}[tb]
  \caption{Performance on molecular property prediction tasks from ECHO-Chem. Results show MSE and MAE for atomic partial charge (CHARGE) and total molecular energy (ENERGY) prediction. LGSM exhibits strong performance on both tasks, demonstrating its applicability beyond synthetic settings for real-world prediction which require capturing non-local atomic interactions.}
  \label{tab:test_results_chem}
  \begin{center}
    \begin{small}
      \begin{sc}
        \resizebox{1.0\columnwidth}{!}{
            \begin{tabular}{lrrrr}
\toprule
\multicolumn{1}{c}{Dataset} & \multicolumn{2}{c}{charge} & \multicolumn{2}{c}{energy} \\
\cmidrule(lr){2-3}
\cmidrule(lr){4-5}
 & MSE $\downarrow$$\,\times 10^{4}$ & MAE $\downarrow$$\,\times 10^{3}$ & MSE $\downarrow$$\,\times 10^{-3}$ & MAE $\downarrow$ \\
\midrule
A-DGN & 1.456 \tiny{$\pm$ 0.032} & 6.543 \tiny{$\pm$ 0.146} & 1.415 \tiny{$\pm$ 0.799} & 12.486 \tiny{$\pm$ 1.621} \\
DRew & 3.669 \tiny{$\pm$ 0.459} & 9.086 \tiny{$\pm$ 0.473} & 1.281 \tiny{$\pm$ 0.733} & 11.325 \tiny{$\pm$ 2.394} \\
GCN & 3.126 \tiny{$\pm$ 0.263} & 8.421 \tiny{$\pm$ 0.512} & 4.561 \tiny{$\pm$ 0.176} & 28.112 \tiny{$\pm$ 1.239} \\
GCNII & 3.490 \tiny{$\pm$ 0.147} & 8.829 \tiny{$\pm$ 0.021} & 1.560 \tiny{$\pm$ 0.653} & 13.235 \tiny{$\pm$ 2.630} \\
GIN & 5.750 \tiny{$\pm$ 0.239} & 10.784 \tiny{$\pm$ 0.059} & 12.215 \tiny{$\pm$ 2.878} & 47.851 \tiny{$\pm$ 10.154} \\
GPS & 1.620 \tiny{$\pm$ 0.065} & 6.182 \tiny{$\pm$ 0.219} & 0.180 \tiny{$\pm$ 0.045} & 5.257 \tiny{$\pm$ 0.842} \\
GRIT & 1.765 \tiny{$\pm$ 0.071} & 7.134 \tiny{$\pm$ 6.090} & 0.512 \tiny{$\pm$ 0.149} & 25.508 \tiny{$\pm$ 2.507} \\
GraphCON & 13.250 \tiny{$\pm$ 0.265} & 19.629 \tiny{$\pm$ 0.195} & 0.975 \tiny{$\pm$ 0.242} & 14.295 \tiny{$\pm$ 0.807} \\
PH-DGN & 2.562 \tiny{$\pm$ 0.144} & 7.915 \tiny{$\pm$ 0.269} & 1.359 \tiny{$\pm$ 0.408} & 16.080 \tiny{$\pm$ 1.123} \\
SWAN & \textbf{1.251 \tiny{$\pm$ 0.029}} & 6.109 \tiny{$\pm$ 0.103} & 2.652 \tiny{$\pm$ 2.257} & 12.629 \tiny{$\pm$ 1.157} \\
\midrule
LGSM (ours) & 1.411 \tiny{$\pm$ 0.038} & \textbf{5.556 \tiny{$\pm$ 0.068}} & \textbf{0.152 \tiny{$\pm$ 0.022}} & \textbf{3.272 \tiny{$\pm$ 0.186}} \\
\bottomrule
\end{tabular}
        }
      \end{sc}
    \end{small}
  \end{center}
  \vskip -0.1in
\end{table}

After investigating the design and theoretical properties of LGSM, we evaluated its empirical performance. We consider the recently proposed ECHO benchmark \citep{miglior_can_2025}, LRIM Graph Benchmark \citep{mathys_lrim_2025} as well as the Peptides tasks from the LRGB collection \citep{dwivedi_long_2022}. On Peptides we find that LGSM performs similarly to standard message-passing, meaning that the impact of processing and information depth seems negligible. This is perhaps not as surprising considering that recent techniques within two layers have achieved strong performances \citep{adamczyk_molecular_2025} leveraging inductive bias. We refer to Appendix \ref{app:experimental details} for more details. Therefore, to properly evaluate if LGSM can provide sufficient information and processing depth, we focus on more recent, specialized benchmarks, which were designed to evaluate these capabilities: ECHO-Synth, ECHO-Chem, and LRIM. ECHO-Synth evaluates models on synthetic graph property prediction tasks consisting of predicting graph diameter (DIAM), node eccentricity (ECC), and single-source shortest paths (SSSP). Note that these tasks should require information depth, especially for graph families which exhibit a high graph diameter. As such the synthetic datasets provide an ideal testbeds to evaluate LGSM's  information depth capabilities. On the other hand, ECHO-Chem provides real-world molecular property prediction tasks, including atomic partial charge prediction (ECHO-Charge) and total molecular energy prediction (ECHO-Energy). These tasks require models to capture non-local atomic interactions. This is particularly interesting because we can evaluate whether our desired model properties and capabilities apply beyond synthetic settings. Moreover, the baselines provided encompass a wide range of common graph learning benchmarks, including classic MPNNs, more sophisticated multi-hop and differential equation inspired message-passing as well as graph transformers. For more details on the details of our training setup and dataset details, we refer to Appendix \ref{app:experimental details}.

\begin{figure*}[t]
    \centering
    \includegraphics[width=0.6\textwidth]{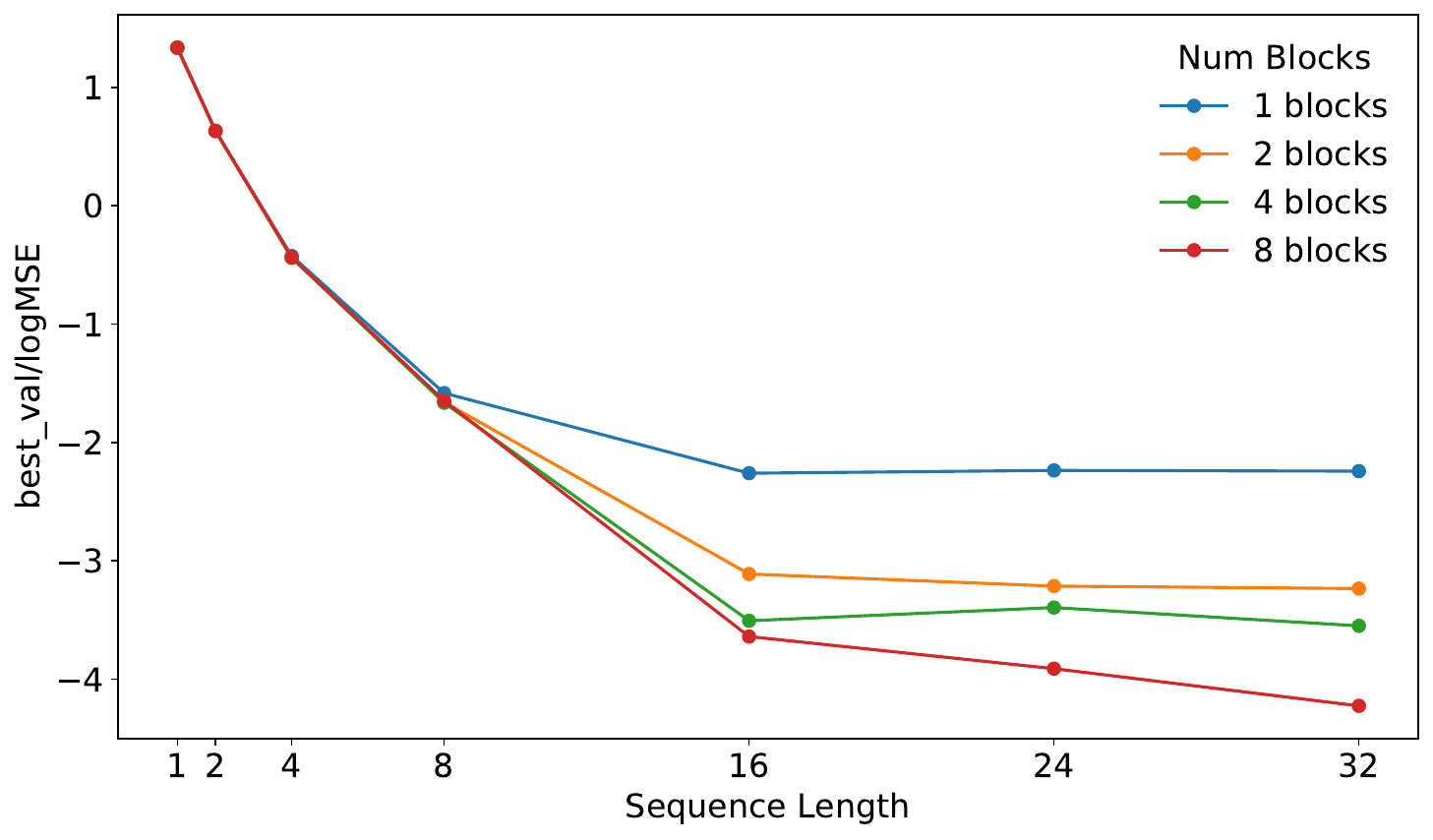}
    \caption{Reporting validation logMSE (lower is better) on the LRIM-16-hard dataset. We ablate the impact of scaling either information depth or processing depth with LGSM. We validate that both increasing the sequence length and the model depth has a positive effect. This underscores the ability of the LGSM architecture to incorporate both axis in a principled and effective manner.}
    \label{fig:lrim_main}
\end{figure*}

Following the recommended setup, we train all our models to minimize logMSE during training and report both the MSE and MAE score metric in Table \ref{tab:test_results_chem} and  Table \ref{tab:test_results_all}, respectively. We report the baseline scores directly from the ECHO benchmark, all our model results are averaged over three random seeds. Note that the models are not trained to minimize MAE directly and only the MSE metric is optimized during training. LGSM surpasses the baselines across most tasks in the ECHO benchmark. The model demonstrates particularly strong performance on the synthetic graph property prediction tasks. In the ablation presented in Figure \ref{fig:eccentricity}, we evaluate the performance of the model with respect to the sequence length. The processing depth and, as a consequence, the parameter budget stay fixed. Nevertheless, we observe that, with increasing sequence length, performance consistently improves. This confirms that the ability to properly handle large information depths through sequence modeling is one of the key elements for the performance of LGSM on these datasets. In addition, the framework achieves good performance on real-world molecular property prediction tasks. Moreover, we evaluate on the LRIM-16-hard dataset, which is a dataset specifically designed to contain provable long-range interactions which require information gathering throughout the graph. In Figure \ref{fig:lrim_main} we observe that when we ablate over the LGSM configurations both the sequence length (information) and the number of blocks (processing) are indeed beneficial for the task, especially since for longer sequences the increase in pure information depth yields diminishing returns. We refer to the Appendix for the complete results on the LRIM Benchmark. This includes an in-depth ablation study on the individual components of LGSM and a discussion on empirical runtime.

\section{Limitations}

While LGSM is designed with computational efficiency in mind, such as linearization and parallel scan algorithms, sequence preprocessing, row or columnwise forward passes to balance resources, we do not systematically explore these computational trade-offs in this work. We also acknowledge that there are other factors which can impact the performance of graph architectures such as regularization or domain inductive-bias besides the information and processing depth. We do not claim our framework to be universally optimal, instead, we aim to provide a theoretically based and principled understanding of the design principles of message-passing-based architectures.

\section{Conclusion}

We introduce Linearized Graph Sequence Models, a framework that views message-passing based graph learning from the perspective of sequence modeling by explicitly decoupling two fundamental dimensions of computation: information depth and processing depth. Unlike standard message-passing networks where these dimensions advance simultaneously, LGSM linearizes propagation along the information axis, enabling information to flow across arbitrary graph distances while confining non-linear transformations to the processing dimension. This separation allows leveraging modern state-space model implementations for efficient parallel computation. Moreover, it enables proper information and gradient flow throughout and as a consequence reframes architectural choices in graph learning as questions about sequence extraction.
Our theoretical analysis fully characterizes the architecture and directly connects node sensitivity to the choice of sequence extraction. Based on this insight, we leverage non-backtracking walk sequences, which offer favorable properties, especially for long-range information propagation. Our empirical evaluation on the ECHO benchmark demonstrates that LGSM achieves strong performance on both synthetic tasks for graph property prediction and on real-world molecular tasks.
Our work demonstrates how the simplification of the architecture and the clear separation of information and processing depth can shift key architectural questions into input modeling choices, providing a principled way for integrating modern deep learning techniques into graph learning.

\bibliography{echo, zoter_cit}

@inproceedings{kipf_semi-supervised_2017,
	title = {Semi-{Supervised} {Classification} with {Graph} {Convolutional} {Networks}},
	url = {https://openreview.net/forum?id=SJU4ayYgl},
	language = {en},
	urldate = {2026-01-27},
	author = {Kipf, Thomas N. and Welling, Max},
	month = feb,
	year = {2017},
}

@inproceedings{chen_simple_2020,
	title = {Simple and {Deep} {Graph} {Convolutional} {Networks}},
	issn = {2640-3498},
	url = {https://proceedings.mlr.press/v119/chen20v.html},
	language = {en},
	urldate = {2026-01-27},
	booktitle = {Proceedings of the 37th {International} {Conference} on {Machine} {Learning}},
	publisher = {PMLR},
	author = {Chen, Ming and Wei, Zhewei and Huang, Zengfeng and Ding, Bolin and Li, Yaliang},
	month = nov,
	year = {2020},
	pages = {1725--1735},
}

@inproceedings{gutteridge_drew_2023,
	address = {Honolulu, Hawaii, USA},
	series = {{ICML}'23},
	title = {{DRew}: dynamically rewired message passing with delay},
	volume = {202},
	shorttitle = {{DRew}},
	urldate = {2026-01-27},
	booktitle = {Proceedings of the 40th {International} {Conference} on {Machine} {Learning}},
	publisher = {JMLR.org},
	author = {Gutteridge, Benjamin and Dong, Xiaowen and Bronstein, Michael and Di Giovanni, Francesco},
	month = jul,
	year = {2023},
	pages = {12252--12267},
}

@article{gravina_oversquashing_2025,
	title = {On {Oversquashing} in {Graph} {Neural} {Networks} {Through} the {Lens} of {Dynamical} {Systems}},
	volume = {39},
	copyright = {Copyright (c) 2025 Association for the Advancement of Artificial Intelligence},
	issn = {2374-3468},
	url = {https://ojs.aaai.org/index.php/AAAI/article/view/33858},
	doi = {10.1609/aaai.v39i16.33858},
	language = {en},
	number = {16},
	urldate = {2026-01-27},
	journal = {Proceedings of the AAAI Conference on Artificial Intelligence},
	author = {Gravina, Alessio and Eliasof, Moshe and Gallicchio, Claudio and Bacciu, Davide and Schönlieb, Carola-Bibiane},
	month = apr,
	year = {2025},
	pages = {16906--16914},
}

@inproceedings{heilig_port-hamiltonian_2024,
	title = {Port-{Hamiltonian} {Architectural} {Bias} for {Long}-{Range} {Propagation} in {Deep} {Graph} {Networks}},
	url = {https://openreview.net/forum?id=03EkqSCKuO},
	language = {en},
	urldate = {2026-01-27},
	author = {Heilig, Simon and Gravina, Alessio and Trenta, Alessandro and Gallicchio, Claudio and Bacciu, Davide},
	month = oct,
	year = {2024},
}

@inproceedings{gravina_anti-symmetric_2022,
	title = {Anti-{Symmetric} {DGN}: a stable architecture for {Deep} {Graph} {Networks}},
	shorttitle = {Anti-{Symmetric} {DGN}},
	url = {https://openreview.net/forum?id=J3Y7cgZOOS},
	language = {en},
	urldate = {2026-01-27},
	author = {Gravina, Alessio and Bacciu, Davide and Gallicchio, Claudio},
	month = sep,
	year = {2022},
}

@inproceedings{rusch_graph-coupled_2022,
	title = {Graph-{Coupled} {Oscillator} {Networks}},
	issn = {2640-3498},
	url = {https://proceedings.mlr.press/v162/rusch22a.html},
	language = {en},
	urldate = {2026-01-27},
	booktitle = {Proceedings of the 39th {International} {Conference} on {Machine} {Learning}},
	publisher = {PMLR},
	author = {Rusch, T. Konstantin and Chamberlain, Ben and Rowbottom, James and Mishra, Siddhartha and Bronstein, Michael},
	month = jun,
	year = {2022},
	pages = {18888--18909},
}

@inproceedings{arroyo_vanishing_2025,
	title = {On {Vanishing} {Gradients}, {Over}-{Smoothing}, and {Over}-{Squashing} in {GNNs}: {Bridging} {Recurrent} and {Graph} {Learning}},
	shorttitle = {On {Vanishing} {Gradients}, {Over}-{Smoothing}, and {Over}-{Squashing} in {GNNs}},
	url = {https://openreview.net/forum?id=N4cyRMuLyl&referrer=%5Bthe%20profile%20of%20Federico%20Barbero%5D(%2Fprofile%3Fid%3D~Federico_Barbero1)},
	language = {en},
	urldate = {2026-01-06},
	author = {Arroyo, Alvaro and Gravina, Alessio and Gutteridge, Benjamin and Barbero, Federico and Gallicchio, Claudio and Dong, Xiaowen and Bronstein, Michael M. and Vandergheynst, Pierre},
	month = oct,
	year = {2025},
}

@misc{ceni_message-passing_2025,
	title = {Message-{Passing} {State}-{Space} {Models}: {Improving} {Graph} {Learning} with {Modern} {Sequence} {Modeling}},
	shorttitle = {Message-{Passing} {State}-{Space} {Models}},
	url = {http://arxiv.org/abs/2505.18728},
	doi = {10.48550/arXiv.2505.18728},
	urldate = {2026-01-06},
	publisher = {arXiv},
	author = {Ceni, Andrea and Gravina, Alessio and Gallicchio, Claudio and Bacciu, Davide and Schonlieb, Carola-Bibiane and Eliasof, Moshe},
	month = may,
	year = {2025},
	note = {arXiv:2505.18728 [cs]},
}

@inproceedings{eliasof_graph_2025,
	title = {Graph {Adaptive} {Autoregressive} {Moving} {Average} {Models}},
	issn = {2640-3498},
	url = {https://proceedings.mlr.press/v267/eliasof25a.html},
	language = {en},
	urldate = {2026-01-06},
	booktitle = {Proceedings of the 42nd {International} {Conference} on {Machine} {Learning}},
	publisher = {PMLR},
	author = {Eliasof, Moshe and Gravina, Alessio and Ceni, Andrea and Gallicchio, Claudio and Bacciu, Davide and Schönlieb, Carola-Bibiane},
	month = oct,
	year = {2025},
	pages = {15232--15265},
}

@inproceedings{ding_recurrent_2024,
	title = {Recurrent {Distance} {Filtering} for {Graph} {Representation} {Learning}},
	issn = {2640-3498},
	url = {https://proceedings.mlr.press/v235/ding24d.html},
	language = {en},
	urldate = {2026-01-06},
	booktitle = {Proceedings of the 41st {International} {Conference} on {Machine} {Learning}},
	publisher = {PMLR},
	author = {Ding, Yuhui and Orvieto, Antonio and He, Bobby and Hofmann, Thomas},
	month = jul,
	year = {2024},
	pages = {11002--11015},
}

@inproceedings{gilmer_neural_2017,
	address = {Sydney, NSW, Australia},
	series = {{ICML}'17},
	title = {Neural message passing for {Quantum} chemistry},
	url = {https://dl.acm.org/doi/10.5555/3305381.3305512},
	urldate = {2026-01-25},
	booktitle = {Proceedings of the 34th {International} {Conference} on {Machine} {Learning} - {Volume} 70},
	publisher = {JMLR.org},
	author = {Gilmer, Justin and Schoenholz, Samuel S. and Riley, Patrick F. and Vinyals, Oriol and Dahl, George E.},
	month = aug,
	year = {2017},
	pages = {1263--1272},
}

@inproceedings{alon_bottleneck_2020,
	title = {On the {Bottleneck} of {Graph} {Neural} {Networks} and its {Practical} {Implications}},
	url = {https://openreview.net/forum?id=i80OPhOCVH2},
	language = {en},
	urldate = {2026-01-25},
	author = {Alon, Uri and Yahav, Eran},
	month = oct,
	year = {2020},
}

@misc{arnaiz-rodriguez_oversmoothing_2025,
	title = {Oversmoothing, {Oversquashing}, {Heterophily}, {Long}-{Range}, and more: {Demystifying} {Common} {Beliefs} in {Graph} {Machine} {Learning}},
	shorttitle = {Oversmoothing, {Oversquashing}, {Heterophily}, {Long}-{Range}, and more},
	url = {http://arxiv.org/abs/2505.15547},
	doi = {10.48550/arXiv.2505.15547},
	urldate = {2026-01-25},
	publisher = {arXiv},
	author = {Arnaiz-Rodriguez, Adrian and Errica, Federico},
	month = jun,
	year = {2025},
	note = {arXiv:2505.15547 [cs]},
}

@inproceedings{di_giovanni_over-squashing_2023,
	address = {Honolulu, Hawaii, USA},
	series = {{ICML}'23},
	title = {On over-squashing in message passing neural networks: the impact of width, depth, and topology},
	volume = {202},
	shorttitle = {On over-squashing in message passing neural networks},
	urldate = {2026-01-25},
	booktitle = {Proceedings of the 40th {International} {Conference} on {Machine} {Learning}},
	publisher = {JMLR.org},
	author = {Di Giovanni, Francesco and Giusti, Lorenzo and Barbero, Federico and Luise, Giulia and Liò, Pietro and Bronstein, Michael},
	month = jul,
	year = {2023},
	pages = {7865--7885},
}

@inproceedings{vaswani_attention_2017,
	title = {Attention is {All} you {Need}},
	volume = {30},
	url = {https://papers.nips.cc/paper_files/paper/2017/hash/3f5ee243547dee91fbd053c1c4a845aa-Abstract.html},
	urldate = {2026-01-25},
	booktitle = {Advances in {Neural} {Information} {Processing} {Systems}},
	publisher = {Curran Associates, Inc.},
	author = {Vaswani, Ashish and Shazeer, Noam and Parmar, Niki and Uszkoreit, Jakob and Jones, Llion and Gomez, Aidan N and Kaiser, Lukasz and Polosukhin, Illia},
	year = {2017},
}

@inproceedings{dao_transformers_2024,
	address = {Vienna, Austria},
	series = {{ICML}'24},
	title = {Transformers are {SSMs}: generalized models and efficient algorithms through structured state space duality},
	volume = {235},
	shorttitle = {Transformers are {SSMs}},
	urldate = {2026-01-25},
	booktitle = {Proceedings of the 41st {International} {Conference} on {Machine} {Learning}},
	publisher = {JMLR.org},
	author = {Dao, Tri and Gu, Albert},
	month = jul,
	year = {2024},
	pages = {10041--10071},
}

@inproceedings{rampasek_recipe_2022,
	title = {Recipe for a {General}, {Powerful}, {Scalable} {Graph} {Transformer}},
	url = {https://openreview.net/forum?id=lMMaNf6oxKM},
	language = {en},
	urldate = {2026-01-25},
	author = {Rampasek, Ladislav and Galkin, Mikhail and Dwivedi, Vijay Prakash and Luu, Anh Tuan and Wolf, Guy and Beaini, Dominique},
	month = oct,
	year = {2022},
}

@inproceedings{ma_graph_2023,
	address = {Honolulu, Hawaii, USA},
	series = {{ICML}'23},
	title = {Graph inductive biases in transformers without message passing},
	volume = {202},
	urldate = {2026-01-25},
	booktitle = {Proceedings of the 40th {International} {Conference} on {Machine} {Learning}},
	publisher = {JMLR.org},
	author = {Ma, Liheng and Lin, Chen and Lim, Derek and Romero-Soriano, Adriana and Dokania, Puneet K. and Coates, Mark and Torr, Philip H.S. and Lim, Ser-Nam},
	month = jul,
	year = {2023},
	pages = {23321--23337},
}

@inproceedings{stoll_generalizable_2025,
	title = {Generalizable {Insights} for {Graph} {Transformers} in {Theory} and {Practice}},
	url = {https://openreview.net/forum?id=ROfYsQ2KNV&referrer=%5Bthe%20profile%20of%20Timo%20Stoll%5D(%2Fprofile%3Fid%3D~Timo_Stoll1)},
	language = {en},
	urldate = {2026-01-25},
	author = {Stoll, Timo and Müller, Luis and Morris, Christopher},
	month = oct,
	year = {2025},
}

@article{tonshoff_walking_2023,
	title = {Walking {Out} of the {Weisfeiler} {Leman} {Hierarchy}: {Graph} {Learning} {Beyond} {Message} {Passing}},
	issn = {2835-8856},
	shorttitle = {Walking {Out} of the {Weisfeiler} {Leman} {Hierarchy}},
	url = {https://openreview.net/forum?id=vgXnEyeWVY},
	language = {en},
	urldate = {2026-01-25},
	journal = {Transactions on Machine Learning Research},
	author = {Tönshoff, Jan and Ritzert, Martin and Wolf, Hinrikus and Grohe, Martin},
	month = apr,
	year = {2023},
}

@inproceedings{behrouz_graph_2024,
	address = {New York, NY, USA},
	series = {{KDD} '24},
	title = {Graph {Mamba}: {Towards} {Learning} on {Graphs} with {State} {Space} {Models}},
	isbn = {979-8-4007-0490-1},
	shorttitle = {Graph {Mamba}},
	url = {https://dl.acm.org/doi/10.1145/3637528.3672044},
	doi = {10.1145/3637528.3672044},
	urldate = {2026-01-25},
	booktitle = {Proceedings of the 30th {ACM} {SIGKDD} {Conference} on {Knowledge} {Discovery} and {Data} {Mining}},
	publisher = {Association for Computing Machinery},
	author = {Behrouz, Ali and Hashemi, Farnoosh},
	month = aug,
	year = {2024},
	pages = {119--130},
}

@inproceedings{chen_learning_2024,
	title = {Learning {Long} {Range} {Dependencies} on {Graphs} via {Random} {Walks}},
	url = {https://openreview.net/forum?id=kJ5H7oGT2M},
	language = {en},
	urldate = {2026-01-25},
	author = {Chen, Dexiong and Schulz, Till Hendrik and Borgwardt, Karsten},
	month = oct,
	year = {2024},
}

@inproceedings{kim_revisiting_2024,
	title = {Revisiting {Random} {Walks} for {Learning} on {Graphs}},
	url = {https://openreview.net/forum?id=SG1R2H3fa1},
	language = {en},
	urldate = {2026-01-25},
	author = {Kim, Jinwoo and Zaghen, Olga and Suleymanzade, Ayhan and Ryou, Youngmin and Hong, Seunghoon},
	month = oct,
	year = {2024},
}

@inproceedings{wang_non-convolutional_2024,
	title = {Non-convolutional graph neural networks.},
	url = {https://openreview.net/forum?id=JDAQwysFOc},
	language = {en},
	urldate = {2026-01-25},
	author = {Wang, Yuanqing and Cho, Kyunghyun},
	month = nov,
	year = {2024},
}

@article{huang_what_2024,
	title = {What {Can} {We} {Learn} from {State} {Space} {Models} for {Machine} {Learning} on {Graphs}?},
	url = {https://openreview.net/forum?id=xAM9VaXZnY},
	language = {en},
	urldate = {2026-01-29},
	author = {Huang, Yinan and Miao, Siqi and Li, Pan},
	month = oct,
	year = {2024},
}

@inproceedings{grover_node2vec_2016,
	address = {New York, NY, USA},
	series = {{KDD} '16},
	title = {node2vec: {Scalable} {Feature} {Learning} for {Networks}},
	isbn = {978-1-4503-4232-2},
	shorttitle = {node2vec},
	url = {https://dl.acm.org/doi/10.1145/2939672.2939754},
	doi = {10.1145/2939672.2939754},
	urldate = {2026-01-29},
	booktitle = {Proceedings of the 22nd {ACM} {SIGKDD} {International} {Conference} on {Knowledge} {Discovery} and {Data} {Mining}},
	publisher = {Association for Computing Machinery},
	author = {Grover, Aditya and Leskovec, Jure},
	month = aug,
	year = {2016},
	pages = {855--864},
}

@inproceedings{perozzi_deepwalk_2014,
	address = {New York, NY, USA},
	series = {{KDD} '14},
	title = {{DeepWalk}: online learning of social representations},
	isbn = {978-1-4503-2956-9},
	shorttitle = {{DeepWalk}},
	url = {https://dl.acm.org/doi/10.1145/2623330.2623732},
	doi = {10.1145/2623330.2623732},
	urldate = {2026-01-29},
	booktitle = {Proceedings of the 20th {ACM} {SIGKDD} international conference on {Knowledge} discovery and data mining},
	publisher = {Association for Computing Machinery},
	author = {Perozzi, Bryan and Al-Rfou, Rami and Skiena, Steven},
	month = aug,
	year = {2014},
	pages = {701--710},
}

@inproceedings{arnaiz-rodriguez_diffwire_2022,
	title = {{DiffWire}: {Inductive} {Graph} {Rewiring} via the {Lovász} {Bound}},
	shorttitle = {{DiffWire}},
	url = {https://proceedings.mlr.press/v198/arnaiz-rodri-guez22a.html},
	language = {en},
	urldate = {2026-01-29},
	booktitle = {Proceedings of the {First} {Learning} on {Graphs} {Conference}},
	publisher = {PMLR},
	author = {Arnaiz-Rodrıiguez, Adrian and Begga, Ahmed and Escolano, Francisco and Oliver, Nuria M.},
	month = dec,
	year = {2022},
	note = {ISSN: 2640-3498},
	pages = {15:1--15:27},
}

@misc{grotschla_benchmarking_2026,
	title = {Benchmarking {Positional} {Encodings} for {GNNs} and {Graph} {Transformers}},
	url = {http://arxiv.org/abs/2411.12732},
	doi = {10.1145/3770854.3785701},
	urldate = {2026-01-29},
	author = {Grötschla, Florian and Xie, Jiaqing and Wattenhofer, Roger},
	month = jan,
	year = {2026},
	note = {arXiv:2411.12732 [cs]},
}

@inproceedings{wu_simplifying_2019,
	title = {Simplifying {Graph} {Convolutional} {Networks}},
	url = {https://proceedings.mlr.press/v97/wu19e.html},
	language = {en},
	urldate = {2026-01-29},
	booktitle = {Proceedings of the 36th {International} {Conference} on {Machine} {Learning}},
	publisher = {PMLR},
	author = {Wu, Felix and Souza, Amauri and Zhang, Tianyi and Fifty, Christopher and Yu, Tao and Weinberger, Kilian},
	month = may,
	year = {2019},
	note = {ISSN: 2640-3498},
	pages = {6861--6871},
}

@inproceedings{cohen_convexified_2021,
	title = {Convexified {Graph} {Neural} {Networks} for {Distributed} {Control} in {Robotic} {Swarms}},
	volume = {3},
	url = {https://www.ijcai.org/proceedings/2021/318},
	doi = {10.24963/ijcai.2021/318},
	language = {en},
	urldate = {2026-01-29},
	author = {Cohen, Saar and Agmon, Noa},
	month = aug,
	year = {2021},
	note = {ISSN: 1045-0823},
	pages = {2307--2313},
}

@inproceedings{keriven_not_2022,
	address = {Red Hook, NY, USA},
	series = {{NIPS} '22},
	title = {Not too little, not too much: a theoretical analysis of graph (over)smoothing},
	isbn = {978-1-7138-7108-8},
	shorttitle = {Not too little, not too much},
	urldate = {2026-01-29},
	booktitle = {Proceedings of the 36th {International} {Conference} on {Neural} {Information} {Processing} {Systems}},
	publisher = {Curran Associates Inc.},
	author = {Keriven, Nicolas},
	month = nov,
	year = {2022},
	pages = {2268--2281},
}

@inproceedings{mathys_new_2025,
	title = {A {New} {Perspective} for {Graph} {Learning} {Architecture} {Design}: {Linearize} {Your} {Depth} {Away}},
	shorttitle = {A {New} {Perspective} for {Graph} {Learning} {Architecture} {Design}},
	url = {https://openreview.net/forum?id=lybxAx4Msq},
	language = {en},
	urldate = {2026-01-29},
	author = {Mathys, Joël and Wattenhofer, Roger},
	month = oct,
	year = {2025},
}

@article{scarselli_graph_2009,
	title = {The {Graph} {Neural} {Network} {Model}},
	volume = {20},
	issn = {1941-0093},
	url = {https://ieeexplore.ieee.org/document/4700287},
	doi = {10.1109/TNN.2008.2005605},
	number = {1},
	urldate = {2026-01-29},
	journal = {IEEE Transactions on Neural Networks},
	author = {Scarselli, Franco and Gori, Marco and Tsoi, Ah Chung and Hagenbuchner, Markus and Monfardini, Gabriele},
	month = jan,
	year = {2009},
	pages = {61--80},
}

@article{micheli_neural_2009,
	title = {Neural network for graphs: a contextual constructive approach},
	volume = {20},
	issn = {1045-9227},
	shorttitle = {Neural network for graphs},
	url = {https://doi.org/10.1109/TNN.2008.2010350},
	doi = {10.1109/TNN.2008.2010350},
	number = {3},
	urldate = {2026-01-29},
	journal = {Trans. Neur. Netw.},
	author = {Micheli, Alessio},
	month = mar,
	year = {2009},
	pages = {498--511},
}

@inproceedings{xu_representation_2018,
	title = {Representation {Learning} on {Graphs} with {Jumping} {Knowledge} {Networks}},
	url = {https://proceedings.mlr.press/v80/xu18c.html},
	language = {en},
	urldate = {2026-01-29},
	booktitle = {Proceedings of the 35th {International} {Conference} on {Machine} {Learning}},
	publisher = {PMLR},
	author = {Xu, Keyulu and Li, Chengtao and Tian, Yonglong and Sonobe, Tomohiro and Kawarabayashi, Ken-ichi and Jegelka, Stefanie},
	month = jul,
	year = {2018},
	note = {ISSN: 2640-3498},
	pages = {5453--5462},
}

@article{park_non-backtracking_2024,
	title = {Non-backtracking {Graph} {Neural} {Networks}},
	issn = {2835-8856},
	url = {https://openreview.net/forum?id=64HdQKnyTc},
	language = {en},
	urldate = {2026-01-29},
	journal = {Transactions on Machine Learning Research},
	author = {Park, Seonghyun and Ryu, Narae and Kim, Gahee and Woo, Dongyeop and Yun, Se-Young and Ahn, Sungsoo},
	month = jun,
	year = {2024},
}

@inproceedings{gu_mamba_2024,
	title = {Mamba: {Linear}-{Time} {Sequence} {Modeling} with {Selective} {State} {Spaces}},
	shorttitle = {Mamba},
	url = {https://openreview.net/forum?id=tEYskw1VY2},
	language = {en},
	urldate = {2026-01-29},
	author = {Gu, Albert and Dao, Tri},
	month = aug,
	year = {2024},
}

@inproceedings{xu_how_2018,
	title = {How {Powerful} are {Graph} {Neural} {Networks}?},
	url = {https://openreview.net/forum?id=ryGs6iA5Km},
	language = {en},
	urldate = {2026-01-29},
	author = {Xu*, Keyulu and Hu*, Weihua and Leskovec, Jure and Jegelka, Stefanie},
	month = sep,
	year = {2018},
}

@inproceedings{gu_efficiently_2021,
	title = {Efficiently {Modeling} {Long} {Sequences} with {Structured} {State} {Spaces}},
	url = {https://openreview.net/forum?id=uYLFoz1vlAC},
	language = {en},
	urldate = {2026-01-27},
	author = {Gu, Albert and Goel, Karan and Re, Christopher},
	month = oct,
	year = {2021},
}

@inproceedings{gu_hippo_2020,
	address = {Red Hook, NY, USA},
	series = {{NIPS} '20},
	title = {{HiPPO}: recurrent memory with optimal polynomial projections},
	isbn = {978-1-7138-2954-6},
	shorttitle = {{HiPPO}},
	url = {https://dl.acm.org/doi/10.5555/3495724.3495849},
	urldate = {2026-01-27},
	booktitle = {Proceedings of the 34th {International} {Conference} on {Neural} {Information} {Processing} {Systems}},
	publisher = {Curran Associates Inc.},
	author = {Gu, Albert and Dao, Tri and Ermon, Stefano and Rudra, Atri and Ré, Christopher},
	month = dec,
	year = {2020},
	pages = {1474--1487},
}

@inproceedings{dwivedi_long_2022,
	address = {Red Hook, NY, USA},
	series = {{NIPS} '22},
	title = {Long range graph benchmark},
	isbn = {978-1-7138-7108-8},
	urldate = {2026-05-06},
	booktitle = {Proceedings of the 36th {International} {Conference} on {Neural} {Information} {Processing} {Systems}},
	publisher = {Curran Associates Inc.},
	author = {Dwivedi, Vijay Prakash and Rampášek, Ladislav and Galkin, Mikhail and Parviz, Ali and Wolf, Guy and Luu, Anh Tuan and Beaini, Dominique},
	month = nov,
	year = {2022},
	pages = {22326--22340},
}

@inproceedings{mathys_lrim_2025,
	title = {{LRIM}: a {Physics}-{Based} {Benchmark} for {Provably} {Evaluating} {Long}-{Range} {Capabilities} in {Graph} {Learning}},
	shorttitle = {{LRIM}},
	url = {https://openreview.net/forum?id=IAZXEX1dVV},
	language = {en},
	urldate = {2026-05-06},
	author = {Mathys, Joël and Christiansen, Henrik and Errica, Federico and Maruyama, Takashi and Alesiani, Francesco},
	month = oct,
	year = {2025},
}

@article{miglior_can_2025,
	title = {Can {You} {Hear} {Me} {Now}? {A} {Benchmark} for {Long}-{Range} {Graph} {Propagation}},
	shorttitle = {Can {You} {Hear} {Me} {Now}?},
	url = {https://openreview.net/forum?id=PcKkRJdX4n},
	language = {en},
	urldate = {2026-05-06},
	author = {Miglior, Luca and Tolloso, Matteo and Gravina, Alessio and Bacciu, Davide},
	month = may,
	year = {2025},
}

@article{tonshoff_where_2024,
	title = {Where {Did} the {Gap} {Go}? {Reassessing} the {Long}-{Range} {Graph} {Benchmark}},
	issn = {2835-8856},
	shorttitle = {Where {Did} the {Gap} {Go}?},
	url = {https://openreview.net/forum?id=Nm0WX86sKv},
	language = {en},
	urldate = {2026-05-06},
	journal = {Transactions on Machine Learning Research},
	author = {Tönshoff, Jan and Ritzert, Martin and Rosenbluth, Eran and Grohe, Martin},
	month = jan,
	year = {2024},
}

@misc{adamczyk_molecular_2025,
	title = {Molecular {Fingerprints} {Are} {Strong} {Models} for {Peptide} {Function} {Prediction}},
	url = {http://arxiv.org/abs/2501.17901},
	doi = {10.48550/arXiv.2501.17901},
	urldate = {2026-05-06},
	publisher = {arXiv},
	author = {Adamczyk, Jakub and Ludynia, Piotr and Czech, Wojciech},
	month = oct,
	year = {2025},
	note = {arXiv:2501.17901 [q-bio]
version: 2},
}
\bibliographystyle{icml2026}

\newpage
\appendix
\onecolumn

\section{Extended Related Work}\label{app:related work}

\textit{Message-passing} With the establishment of message-passing \citep{micheli_neural_2009, scarselli_graph_2009} as the core component to incorporate the graph bias, several different architectures have been proposed \citep{xu_how_2018, gilmer_neural_2017}. Most of them consider message-passing on the adjacency or a normalized version of it \citep{kipf_semi-supervised_2017}. As such, all share the inherent mechanism that leads to loss of information due to repeated deep aggregations \citep{alon_bottleneck_2020, di_giovanni_over-squashing_2023, arnaiz-rodriguez_oversmoothing_2025}. This has motivated considering complementary mechanisms such as virtual compute structures \citet{arnaiz-rodriguez_diffwire_2022} or entirely different approaches, such as graph transformers \citet{ma_graph_2023, stoll_generalizable_2025}. However, in particular the latter often requires additional modeling of inductive bias by positional embeddings \citep{grotschla_benchmarking_2026} or re-incorporating in the message-passing \citep{rampasek_recipe_2022}. Alternatively, the connection to discretizing differential equations has led to insights for mitigating vanishing node sensitivity and proper optimization by restricting the weight spaces \citep{gravina_anti-symmetric_2022, gravina_oversquashing_2025, heilig_port-hamiltonian_2024, arroyo_vanishing_2025}. Finally, different linearization approaches have been considered for mathematical analysis \citep{keriven_not_2022}, ease of optimization \citep{cohen_convexified_2021} or efficiency considerations \citep{ceni_message-passing_2025, wu_simplifying_2019, mathys_new_2025}.

\textit{Non-backtracking} Walk based methods \citep{grover_node2vec_2016, perozzi_deepwalk_2014} for graph learning have seen a resurgence in recent years \citep{behrouz_graph_2024, chen_learning_2024, wang_non-convolutional_2024, kim_revisiting_2024}, due to leveraging powerful sequence processing architectures. Most of these make use of non-backtracking walks and further employ some kind of local memory to incorporate neighboring nodes and edges as sliding windows \citep{tonshoff_walking_2023}.
Moreover, modeling the same behavior for message-passing architectures by explicitly changing the aggregation to be edge state dependent has favorable properties, such as better sensitivity and spectral properties \citep{park_non-backtracking_2024}. 

\section{Detailed Architecture}

A detailed specification of the LGSM architecture is provided in \Cref{alg:detailed_architecture}.
The model operates on graph-structured data by first transforming the input graph features
$\mathbf X \in \R^{d \times n}$ and adjacency matrix $\mAg \in \R^{n\times n}$ into a latent 
sequence representation. Specifically, a sequence extractor ($\textsc{Seq}$) maps the inputs
$(\mathbf X, \mAg)$ to an initial sequence $\smash{\mathbf S^{(0)}_{\text{in}}, \ldots, \mathbf S^{(L-1)}_{\text{in}}}$.
This sequence is subsequently processed through a stack of $D$ blocks, each comprising three components: a State-Space Model 
($\textsc{SSM}$), a Feed-Forward Network ($\sigma_1$), and a Graph Mixing layer.

In the first step of each layer, the SSM processes the input sequence $\mathbf S^{(0)}_{\text{in}}, \ldots, \mathbf S^{(L-1)}_{\text{in}}$ to an output sequence. The output is stabilized via a residual 
connection and layer normalization ($\textsc{LayerNorm}$) to obtain $\smash{\mathbf S^{(0)}_{\text{SSM}}, \ldots, \mathbf S^{(L-1)}_{\text{SSM}}}$. Subsequently, the sequence are passed
through a element-wise non-linear FFN, followed by a second residual connection and normalization step
to produce $\smash{\mathbf S^{(0)}_{\text{FFN}}, \ldots, \mathbf S^{(L-1)}_{\text{FFN}}}$. Finally, the Graph Mixing layer updates the sequence elements with the graph-propagated features of the preceding element.
This combined representation is processed by a second non-linearity ($\sigma_2$), refined with 
a skip connection, and normalized.

\begin{algorithm}
    \begin{algorithmic}[1]
        \STATE \textbf{Input: } $\mathbf X \in \R^{d \times n}, \mathbf A_G \in \R^{n\times n}, L \in \N, D\in \N$
        \STATE \textbf{Output: } $\mathbf S_{\text{out}} \in \R^{d \times n}$
        \STATE $\triangleright$ Initial sequence extraction
        \STATE $\mathbf S_{\text{in}}^{(0)}, \ldots, \mathbf S_{\text{in}}^{(L-1)} \gets \textsc{Seq}(\mathbf X, \mathbf A_G)$
        \FOR{$d = 1,\ldots,D$}
        \STATE $\triangleright$ SSM Layer
        \STATE $\mathbf S_{\text{SSM}}^{(0)}, \ldots, \mathbf S_{\text{SSM}}^{(L-1)} \gets \textsc{SSM}(\mathbf S_{\text{in}}^{(0)}, \ldots, \mathbf S_{\text{in}}^{(L-1)})$
        \FOR{$\ell = 0, \ldots,L-1$}
        \STATE $\mathbf S_{\text{SSM}}^{(\ell)} \gets \textsc{LayerNorm}(\mathbf S^{(\ell)}_{\text{SSM}} + \mathbf S^{(\ell)}_{\text{in}})$
        \STATE $\triangleright$ FFN layer
        \STATE $\mathbf S^{(\ell)}_{\text{FFN}} \gets \sigma_1(\mS^{(\ell)}_{\text{SSM}})$
        \STATE $\mathbf S^{(\ell)}_{\text{FFN}} \gets \textsc{LayerNorm}(\mathbf S^{(\ell)}_{\text{FFN}} + \mS^{(\ell)}_{\text{SSM}})$
        \ENDFOR
        \STATE $\triangleright$ Graph mixing layer
        \STATE $\mathbf S_{\text{MIX}}^{(0)} \gets \mathbf S_{\text{FFN}}^{(0)}$
        \FOR{$\ell = 0 , \ldots, L-1$}
        \STATE $\mathbf S_{\text{MIX}}^{(\ell)} \gets \sigma_2(\mathbf S_{\text{FFN}}^{(\ell)} + \mathbf A_G\mathbf S_{\text{FFN}}^{(\ell-1)})$
        \STATE $\mS^{(\ell)}_{\text{MIX}} \gets \textsc{LayerNorm}(\mS^{(\ell)}_{\text{MIX}} + \mS^{(\ell)}_{\text{FFN}})$
        \ENDFOR
        \STATE $\mS^{(0)}_{\text{in}}, \ldots, \mS^{(L-1)}_{\text{in}}\gets \mS^{(0)}_{\text{MIX}}, \ldots, \mS^{(L-1)}_{\text{MIX}}$
        \ENDFOR
        \STATE $\triangleright$ Final graph decoder
        \STATE $\mS_{\text{out}} \gets \textsc{Dec}(\mathbf S_{\text{in}}^{(L-1)})$
        \STATE \textbf{return } $\mS_{\text{out}}$
    \end{algorithmic}
\caption{Pseudocode of LGSM Architecture}
\label{alg:detailed_architecture}
\end{algorithm}

\clearpage

\section{Theoretical Results}

Let $\a_t := \mathbf C \sum_{k=0}^t \mathbf A^k \mathbf B \mathbf S_{\text{in}}^{(t-k)}$ be the closed-form expression for $\mathbf Y^{(t)}$ in the SSM recurrence \eqref{eq:ssm_rec_1}-\eqref{eq:ssm_rec_2} with initial condition $\mathbf X^{(-1)} = 0$, which follows by induction on $t$:

\begin{align}
    \mathbf X^{(t)} &= \mathbf A\mathbf X^{(t-1)} + \mathbf B\mathbf S^{(t)}_{\text{in}} \label{eq:ssm_rec_1}\\ 
    \mathbf Y^{(t)} &= \mathbf C\mathbf X^{(t)} \label{eq:ssm_rec_2}
\end{align}

In particular, $\a_t = \mathbf Y^{(t)}$. We begin with a helper lemma that bounds the sensitivity of a single SSM block in the absence of graph mixing. This will serve as a building block for the subsequent proofs.

\begin{lemma}\label{lemma:sens_helper}
    Consider a $1$-block LGSM of sequence length $L$ with an SSM defined by the system matrices $(\mathbf A, \mathbf B, \mathbf C)$, input features $\x_v$ and output features $\mathbf{S}_{\text{out}}$. Suppose that the FFN $\sigma_{\text{SSM}}$
    is $\mu_{\text{SSM}}$-regular, and that $\mathbf A$ is normal with spectral radius $\rho(\mathbf A) \approx 1$. Then:
    \begin{align*}
        \left\|\frac{\partial \sigma_{\text{SSM}}(\a_L)}{\partial \x_v}\right\|
        \le 
        \gamma \sum_{k=0}^{L} \left\|\frac{\partial\mathbf S_{\text{in}}^{(k)}}{\partial \x_v}\right\|
    \end{align*}
    Where $\gamma := \mu_{\text{SSM}}\|\mathbf B\|\|\mathbf C\|$.
\end{lemma}

\begin{proof}
    By a direct computation:
    \begin{align*}
        \left\|\frac{\partial \sigma_{\text{SSM}}(\a_L)}{\partial \x_v}\right\| 
        &\le \left\|\frac{\partial \sigma_{\text{SSM}}(\a_L)}{\partial \a_L}\right\|\left\|\frac{\partial \a_L}{\partial \x_v}\right\|  && (\text{submultiplicativity of }\|\cdot\|)\\ 
        &\le \mu_{\text{SSM}}\left\|\frac{\partial \a_L}{\partial \x_v}\right\|  && (\mu_{\text{SSM}}\text{-regularity of }\sigma_{\text{SSM}})\\ 
        &\le \mu_{\text{SSM}}\left\|\frac{\partial}{\partial \x_v} \mathbf C\sum_{k=0}^L \mathbf A^k\mathbf B\mathbf S_{\text{in}}^{(L-k)}\right\| && (\text{def. }\a_L)\\ 
        &\le \mu_{\text{SSM}}\left\|\mathbf C\sum_{k=0}^L \mathbf A^k\mathbf B\frac{\partial\mathbf S_{\text{in}}^{(L-k)}}{\partial\x_v}\right\| && (\mathbf A, \mathbf B, \mathbf C \text{ independent of }\x_v)\\ 
        &\le \mu_{\text{SSM}}\|\mathbf C\|\sum_{k=0}^L \left\|\mathbf A^k\mathbf B\frac{\partial \mathbf S_{\text{in}}^{(L-k)}}{\partial \x_v}\right\| && (\text{submultiplicativity and triangle inequality})\\ 
        &\le \mu_{\text{SSM}}\|\mathbf C\|\sum_{k=0}^L \|\mathbf A^k\|\|\mathbf B\|\left\|\frac{\partial \mathbf S_{\text{in}}^{(L-k)}}{\partial \x_v}\right\| && (\text{submultiplicativity})\\ 
        &\le \mu_{\text{SSM}}\|\mathbf C\|\|\mathbf B\|\sum_{k=0}^L\rho(\mathbf A)^k\left\|\frac{\partial \mathbf S_{\text{in}}^{(L-k)}}{\partial \x_v}\right\| && (\mathbf A\text{ normal})\\ 
        &\le \mu_{\text{SSM}}\|\mathbf C\|\|\mathbf B\|\sum_{k=0}^L\left\|\frac{\partial \mathbf S_{\text{in}}^{(L-k)}}{\partial \x_v}\right\|&& (\rho(\mathbf A) \approx 1)\\ 
        &\le \gamma\sum_{k=0}^L\left\|\frac{\partial \mathbf S_{\text{in}}^{(k)}}{\partial \x_v}\right\| && (\text{re-indexing}, k \mapsto L-k)
    \end{align*}
    Which concludes this proof.
\end{proof}

\clearpage

We now extend this result to a full $1$-block LGSM by additionally incorporating the graph mixing layer.
We use this lemma to show the following theorems. Next we show:

\sensOneLayer*

\begin{proof}
Let $\beta_L := \sigma_\SSM(\a_L) + \mathbf A_\G\sigma_\SSM(\a_{L-1})$ denote the value of the LGSM output before the application of the 
FFN $\sigma_\MIX$, such that $\mathbf S_{\text{out}}^{(L)} = \sigma_\MIX(\b_L)$. Then:

\begin{align*}
    \left\|\frac{\partial\mathbf S_{\text{out}}^{(L)}}{\partial\x_v}\right\|
    &\le \left\|\frac{\partial \sigma_\MIX(\beta_L)}{\partial \beta_L}\frac{\partial \beta_L}{\partial \x_v}\right\| && (\text{def. $1$-block LGSM})\\ 
    &\le \left\|\frac{\partial \sigma_\MIX(\beta_L)}{\partial \beta_L}\right\|\left\|\frac{\partial \beta_L}{\partial \x_v}\right\| && (\text{submultiplicativity})\\ 
    &\le \mu_\MIX\left\|\frac{\partial \beta_L}{\partial \x_v}\right\| && (\mu_\MIX\text{-regularity of }\sigma_\MIX)\\ 
    &\le \mu_\MIX\left\|\frac{\partial \sigma_\SSM(\a_L)}{\partial \x_v} + \mathbf A_\G\frac{\partial\sigma_\SSM(\a_{L-1})}{\partial \x_v}\right\| && (\text{def. } \beta_L)\\ 
    &\le \mu_\MIX\left(\left\|\frac{\partial \sigma_\SSM(\a_L)}{\partial \x_v}\right\| + \|\mathbf A_\G\|\left\|\frac{\partial\sigma_\SSM(\a_{L-1})}{\partial \x_v}\right\|\right) && (\text{triangle ineq. and submult.})\\ 
    &\le \mu_\MIX\left(\gamma'\sum_{k=0}^L\left\|\frac{\partial \mathbf S_{\text{in}}^{(k)}}{\partial \x_v}\right\| + \|\mathbf A_\G\|\gamma'\sum_{k=0}^{L-1}\left\|\frac{\partial \mathbf S_{\text{in}}^{(k)}}{\partial \x_v}\right\|\right) && (\text{\Cref{lemma:sens_helper}})\\ 
    &\le \gamma\sum_{k=0}^L\left\|\frac{\partial \mathbf S_{\text{in}}^{(k)}}{\partial \x_v}\right\| + \gamma\|\mathbf A_\G\|\sum_{k=0}^{L-1}\left\|\frac{\partial \mathbf S_{\text{in}}^{(k)}}{\partial \x_v}\right\| && (\gamma := \mu_\MIX \gamma')
\end{align*}
\end{proof}

Before extending to a general $D$-block LGSM, we first analyze the sensitivity of a $D$-block LGSM without 
graph mixing layers. That is, a $D$-fold stack of alternating SSM layers and an FFN $\sigma_\MIX$.

\sensNoMix*

\begin{proof}
Define the shorthand

\begin{align*}
    G(L,D) := \left\|\frac{\partial \mathbf S_{\text{out}}^{(L,D)}}{\partial \x_v}\right\|
\end{align*}

By \Cref{lemma:sens_helper} and the identity $\mathbf S_{\text{in}}^{(k,D)} = \mathbf S_{\text{out}}^{(k,D-1)}$, 
for $D > 0$ we obtain the recurrence

\begin{align}
    G(L,D) \le \gamma \sum_{k=0}^L G(k,D-1) \label{eq:rec_1}
\end{align}

where $\gamma = \mu_\SSM \|\mathbf B\|\|\mathbf C\|$.

\paragraph{Base cases.} If $D = 0$ we apply no SSM blocks, thus we have $\mathbf S_{\text{out}}^{(L,0)} = \mathbf S_{\text{in}}^{(L,1)}$ and

\begin{align}
    G(L,0) = \left\|\frac{\partial \mathbf S_{\text{in}}^{(L,1)}}{\partial \x_v}\right\|
    \le \gamma^0 \sum_{k=0}^L \binom{L-k-1}{-1} \left\|\frac{\partial \mathbf S_{\text{in}}^{(k,1)}}{\partial \x_v}\right\| \label{eq:rec_11}
\end{align}

Let $L = 0$ and $D > 0$. By repeatedly applying \eqref{eq:rec_1} and then \eqref{eq:rec_11} we obtain:

\begin{align*}
    G(0,D) \le \gamma^D G(0,0) \le \gamma^D \sum_{k=0}^0 \binom{0-k-1}{-1} \left\|\frac{\partial \mathbf S_{\text{in}}^{(0,1)}}{\partial \x_v}\right\| =  \gamma^D \sum_{k=0}^0 \binom{0-k-1}{D-1} \left\|\frac{\partial \mathbf S_{\text{in}}^{(0,1)}}{\partial \x_v}\right\|
\end{align*}

\paragraph{Inductive step.} Let $D > 0, L > 0$, then:

\begin{align*}
    G(L,D) &\le \gamma \sum_{k=0}^L G(k,D-1) && \eqref{eq:rec_1}\\ 
    &\le \gamma \sum_{k=0}^L \gamma^{D-1} \sum_{j=0}^k \binom{k-j+D-2}{D-2} \left\|\frac{\partial \mathbf S_{\text{in}}^{(j,1)}}{\partial \x_v}\right\| && (\textbf{IH})\\ 
    &\le \gamma^D \sum_{0 \le j \le k \le L} \binom{k-j+D-2}{D-2} \left\|\frac{\partial \mathbf S_{\text{in}}^{(j,1)}}{\partial \x_v}\right\|\\ 
    &\le \gamma^D \sum_{j=0}^L  \left\|\frac{\partial \mathbf S_{\text{in}}^{(j,1)}}{\partial \x_v}\right\| \sum_{k = j}^L \binom{k-j+D-2}{D-2}\\ 
    &\le \gamma^D \sum_{j=0}^L  \left\|\frac{\partial \mathbf S_{\text{in}}^{(j,1)}}{\partial \x_v}\right\| \sum_{k = 0}^{L-j} \binom{k+D-2}{D-2}\\ 
    &\le \gamma^D \sum_{j=0}^L  \left\|\frac{\partial \mathbf S_{\text{in}}^{(j,1)}}{\partial \x_v}\right\| \binom{L-j+D-1}{D-1} && (\text{Hockey-Stick Lemma})\\ 
    &\le \gamma^D \sum_{j=0}^L  \binom{L-j+D-1}{D-1} \left\|\frac{\partial \mathbf S_{\text{in}}^{(j,1)}}{\partial \x_v}\right\|
\end{align*}

Establishing the claim.
\end{proof}

\clearpage

Finally we generalize the statement to general $D$-block LGSM.

\begin{restatable}[Local Sensitivity for entire LGSM]{theorem}{sens}
    \label{thm:sens}
    Consider a $D$-block LGSM of sequence length $L$, with an SSM defined by system matrices
    $(\mathbf A, \mathbf B, \mathbf C)$. Denote the input and output features by 
    $\mathbf X, \mathbf S_{\text{out}}^{(L,D)} \in \R^{n\times d}$, respectively, where 
    $\mathbf x_v = (\mathbf X)_{v,:}$, is the feature vector of vertex $v \in \V$.
    Suppose that the FFNs $\sigma_{\SSM}, \sigma_{\MIX}$ are 
    $\mu_\SSM$- and $\mu_\MIX$-regular, respectively, and that $\mathbf A$ is normal with 
    spectral radius $\rho(\mathbf A) \approx 1$. Then, for any vertex $v \in \V$:
    \begin{align*}
        \left\|\frac{\partial \mathbf S^{(L,D)}_{\text{out}}}{\partial \mathbf x_v}\right\| \le \gamma^D \sum_{k=0}^L \left\|\frac{\partial \mathbf S^{(L-k,1)}_{\text{in}}}{\partial \mathbf x_v}\right\|
        \sum_{m=0}^{\min\{k,D\}} \binom{k-m+D-1}{D-1} \binom{D}{m}\|\mAg\|^m
    \end{align*}
    where $\gamma$ is a constant in $\|\mathbf B\|, \|\mathbf C\|, \mu_{\textup{SSM}}$, and $\mu_\MIX$.
\end{restatable}

\begin{proof}

We use $G(L,D) := \|\partial \mathbf S_{\text{out}}^{(L,D)} / \partial \x_v\|$ from earlier. By \Cref{thm:sens_1gsm} we have:

\begin{align}
    G(L,D)
    &\le \gamma\sum_{k=0}^L G(k,D-1) + \gamma\|\mathbf A_\G\|\sum_{k=0}^{L-1} G(k, D-1) \label{eq:rec_2}
\end{align}

\paragraph{Base cases.} If $D = 0$ we apply no SSM blocks, thus we have $\mathbf S_{\text{out}}^{(L,0)} = \mathbf S_{\text{in}}^{(L,1)}$ and 

\begin{align}
    G(L,0) = \left\|\frac{\partial \mathbf S_{\text{in}}^{(L,1)}}{\partial \x_v}\right\|
    \le \gamma^0 \sum_{k=0}^L\left\|\frac{\partial \mathbf S_{\text{in}}^{(L-k,1)}}{\partial \x_v}\right\| \sum_{m=0}^{\min\{k,0\}} \binom{k-m-1}{-1} \binom{0}{m}\|\mathbf A_{\mathcal G}\|^m\label{eq:rec_21}
\end{align}

Let $L = 0$ and $D > 0$. By repeatedly applying \eqref{eq:rec_2} and then \eqref{eq:rec_21} we obtain:

\begin{align*}
    G(0,D) \le \gamma^D G(0,0) \le  
    \gamma^D \sum_{k=0}^0\left\|\frac{\partial \mathbf S_{\text{in}}^{(0-k,1)}}{\partial \x_v}\right\| \sum_{m=0}^{\min\{k,D\}} \binom{k-m+D-1}{D-1} \binom{D}{m}\|\mathbf A_{\mathcal G}\|^m
\end{align*}

\paragraph{Inductive step.} Let $D > 0, L > 0$, then:

\begin{align*}
    G(L,D)
    &\le \gamma\sum_{k=0}^L G(k,D-1) + \gamma\|\mAg\|\sum_{k=0}^{L-1} G(k, D-1) && \eqref{eq:rec_2}\\
    &\le \gamma^D\sum_{k=0}^L \sum_{j=0}^{k} \left\|\frac{\partial \mathbf S_{\text{in}}^{(k-j, 1)}}{\partial \x_v}\right\| \sum_{m=0}^{\min\{j,D-1\}} \binom{j-m+D-2}{D-2}\binom{D-1}{m} \|\mAg\|^m \\
    &\quad + \gamma^D\|\mAg\|\sum_{k=0}^{L-1} \sum_{j=0}^{k} \left\|\frac{\partial \mathbf S_{\text{in}}^{(k-j, 1)}}{\partial \x_v}\right\| \sum_{m=0}^{\min\{j,D-1\}} \binom{j-m+D-2}{D-2}\binom{D-1}{m} \|\mAg\|^m && (\textbf{IH})\\
    &= \gamma^D\sum_{k=0}^L \sum_{j=0}^{k} \left\|\frac{\partial \mathbf S_{\text{in}}^{(k-j, 1)}}{\partial \x_v}\right\| \sum_{m=0}^{\min\{j,D-1\}} \binom{j-m+D-2}{D-2}\binom{D-1}{m} \|\mAg\|^m  \displaybreak[3]\\
    &\quad + \gamma^D\sum_{k=1}^{L} \sum_{j=0}^{k-1} \left\|\frac{\partial \mathbf S_{\text{in}}^{(k\!-\!1\!-\!j, 1)}}{\partial \x_v}\right\| \sum_{m=0}^{\min\{j,D-1\}} \binom{j-m+D-2}{D-2}\binom{D-1}{m} \|\mAg\|^{m+1} && (k \!\mapsto\! k\!-\!1)\\
    &= \gamma^D\sum_{k=0}^L \sum_{j=0}^{k} \left\|\frac{\partial \mathbf S_{\text{in}}^{(k-j, 1)}}{\partial \x_v}\right\| \sum_{m=0}^{\min\{j,D-1\}} \binom{j-m+D-2}{D-2}\binom{D-1}{m} \|\mAg\|^m \\
    &\quad + \gamma^D\sum_{k=1}^{L} \sum_{j=1}^{k} \left\|\frac{\partial \mathbf S_{\text{in}}^{(k-j, 1)}}{\partial \x_v}\right\| \sum_{m=0}^{\min\{j-1,D-1\}} \binom{j\!-\!m\!+\!D\!-\!3}{D-2}\binom{D-1}{m} \|\mAg\|^{m+1} && (j \!\mapsto\! j\!-\!1)\\
    &= \gamma^D\sum_{k=0}^L \sum_{j=0}^{k} \left\|\frac{\partial \mathbf S_{\text{in}}^{(k-j, 1)}}{\partial \x_v}\right\| \sum_{m=0}^{\min\{j,D-1\}} \binom{j-m+D-2}{D-2}\binom{D-1}{m} \|\mAg\|^m \\
    &\quad + \gamma^D\sum_{k=1}^{L} \sum_{j=1}^{k} \left\|\frac{\partial \mathbf S_{\text{in}}^{(k-j, 1)}}{\partial \x_v}\right\| \sum_{m=1}^{\min\{j,D\}} \binom{j-m+D-2}{D-2}\binom{D-1}{m-1} \|\mAg\|^{m} && (m \!\mapsto\! m\!-\!1)\\
    &\le \gamma^D\sum_{k=0}^L \sum_{j=0}^{k} \left\|\frac{\partial \mathbf S_{\text{in}}^{(k-j, 1)}}{\partial \x_v}\right\| \sum_{m=0}^{\min\{j,D-1\}} \binom{j-m+D-2}{D-2}\binom{D-1}{m} \|\mAg\|^m \\
    &\quad + \gamma^D\sum_{k=0}^{L} \sum_{j=0}^{k} \left\|\frac{\partial \mathbf S_{\text{in}}^{(k-j, 1)}}{\partial \x_v}\right\| \sum_{m=0}^{\min\{j,D\}} \binom{j-m+D-2}{D-2}\binom{D-1}{m-1} \|\mAg\|^{m} && (\text{extend; } \tbinom{D-1}{-1}\!=\!0)\\
    &= \gamma^D \sum_{k=0}^L \sum_{j=0}^k \left\|\frac{\partial \mathbf S_{\text{in}}^{(k-j, 1)}}{\partial \x_v}\right\| \sum_{m=0}^{\min\{j,D\}} \binom{j\!-\!m\!+\!D\!-\!2}{D-2}\left[\binom{D\!-\!1}{m} + \binom{D\!-\!1}{m\!-\!1}\right] \|\mAg\|^m && (\text{merge; } \tbinom{D\!-\!1}{D}\!=\!0)\\
    &= \gamma^D \sum_{k=0}^L \sum_{j=0}^k \left\|\frac{\partial \mathbf S_{\text{in}}^{(k-j, 1)}}{\partial \x_v}\right\| \sum_{m=0}^{\min\{j,D\}} \binom{j-m+D-2}{D-2}\binom{D}{m} \|\mAg\|^m && (\text{Pascal's identity})\\
    &= \gamma^D \sum_{k=0}^L \left\|\frac{\partial \mathbf S_{\text{in}}^{(L\!-\!k, 1)}}{\partial \x_v}\right\| \sum_{m=0}^{\min\{k,D\}} \binom{D}{m}\|\mAg\|^m \sum_{s=0}^{k-m} \binom{s+D-2}{D-2} && (\text{re-index; swap sums})\\
    &= \gamma^D \sum_{k=0}^L \left\|\frac{\partial \mathbf S_{\text{in}}^{(L-k, 1)}}{\partial \x_v}\right\| \sum_{m=0}^{\min\{k,D\}} \binom{k\!-\!m\!+\!D\!-\!1}{D-1}\binom{D}{m} \|\mAg\|^m && (\text{Hockey-Stick})
\end{align*}

Establishing the claim.
\end{proof}

\clearpage

\paragraph{Remark.} We want to comment on the assumptions made in \Cref{thm:sens_1gsm}, \Cref{thm:sens_noGM}, and \Cref{thm:sens}.
The $\mu_\SSM$- and $\mu_\MIX$-regularity of $\sigma_\SSM$ and $\sigma_\MIX$ respectively, are justified since we initialize the
$\sigma_\SSM$ and $\sigma_\MIX$ FFN with MLPs consisting of dense hidden layers with GELU activations.
Since GELU is smooth and Lipschitz on bounded domains, each such MLP has bounded Jacobian norm,
ensuring $\mu$-regularity for a finite constant $\mu > 0$ determined by the network weights.

The assumption that $\mathbf A$ is normal with spectral radius $\rho(\mathbf A) \approx 1$ is standard in
structured SSM architectures. In practice, $\mathbf A$ is initialized as a diagonal matrix (e.g., via the
HiPPO framework), which is trivially normal. The condition $\rho(\mathbf A) \approx 1$ reflects the
design choice of placing eigenvalues near the unit circle to enable long-range information propagation
without exponential growth or decay of the hidden state across sequence steps.

\section{Dataset Details} \label{app:dataset details}

\texttt{ECHO}~\citep{miglior_can_2025} (Evaluating Communication over long HOps) is a comprehensive benchmarke designed to evaluate the long-range modeling capabilities of GNNs. The benchmark is categorized into two collections of datasets \texttt{ECHO-Synth} and \texttt{ECHO-Chem}.
\texttt{ECHO-Synth} evaluates GNNs on the synthetic graph property prediction tasks \texttt{ECHO-DIAM}
graph diameter, \texttt{ECHO-ECC} node eccentricity and
\texttt{ECHO-SSSP} single-source shortest path.
\texttt{ECHO-Chem} provides real-world challenges
derived from computational chemistry, requiring 
GNNs to capture non-local atomic interactions.
The dataset \texttt{ECHO-Charge} is focused on predicting atomic charge distributions while 
\texttt{ECHO-Energy} is centered around predicting total molecular energy. Statistics for all datasets used are listed in \Cref{tab:dataset_properties} and \ref{tab:dataset_statistics}.

\begin{table}[h]
\centering
\small
\caption{Summary of ECHO datasets, taken from \citet{miglior_can_2025}.}
\vspace{1mm}
\resizebox{\textwidth}{!}{%
\begin{tabular}{lccc}
\toprule
\textbf{Dataset} & \textbf{Node Features} & \textbf{Edge Features} & \textbf{Target} \\
\midrule
\texttt{ECHO-Synth} & Random scalar, source indicator for \texttt{sssp} & None & \texttt{diam}, \texttt{sssp}, \texttt{ecc} \\
\texttt{ECHO-Chem} & Atomic number, distance from center of mass & Bond type, bond length & Partial charges, Total energy \\
\bottomrule
\end{tabular}
}
\label{tab:dataset_properties}
\end{table}

\begin{table}[h]
    \centering
    \scriptsize
    \setlength{\tabcolsep}{4pt}
    \caption{Statistics of the ECHO datasets, taken from.}
    \begin{tabular}{lrrrrrrrr}
        \toprule
        \textbf{Dataset} & \textbf{\# Graphs} & \textbf{Avg Nodes} & \textbf{Avg Deg.} & \textbf{Avg Edges} &  \textbf{Avg Diam} & \textbf{\# Node Feat} & \textbf{\# Edge Feat} & \textbf{\# Tasks} \\
        \midrule
        \texttt{ECHO-Synth} & 10,080 & 83.69$_{\pm66.24}$ & 2.53$_{\pm1.19}$ & 211.63$_{\pm209.39}$ & 28.50$_{\pm6.92}$ & 2 & None & 3 \\
        \midrule
        \quad \texttt{line} & 1,680 & 75.60$_{\pm27.32}$ & 2.37$_{\pm0.10}$ & 90.10$_{\pm33.89}$ & 28.50$_{\pm6.92}$ & 2 & None & 3\\
        \quad \texttt{ladder} & 1,680 & 56.52$_{\pm13.82}$ & 2.92$_{\pm0.02}$ & 82.54$_{\pm20.72}$ & 28.50$_{\pm6.92}$ & 2 & None & 3\\
        \quad \texttt{grid} & 1,680 & 193.10$_{\pm93.10}$ & 2.95$_{\pm0.12}$ & 288.32$_{\pm145.29}$ & 28.50$_{\pm6.92}$ & 2 & None & 3\\
        \quad \texttt{tree} & 1,680 & 60.42$_{\pm17.17}$ & 1.96$_{\pm0.01}$ & 59.42$_{\pm17.17}$ & 28.50$_{\pm6.92}$ & 2 & None & 3\\ 
        \quad \texttt{caterpillar} & 1,680 & 34.71$_{\pm7.96}$ & 1.94$_{\pm0.02}$ & 33.71$_{\pm7.96}$ & 28.50$_{\pm6.92}$ & 2 & None & 3\\
        \quad \texttt{lobster} & 1,680 & 81.79$_{\pm25.46}$ & 1.97$_{\pm0.01}$ & 80.79$_{\pm25.46}$ & 28.50$_{\pm6.92}$ & 2 & None & 3\\
        \midrule
        \texttt{ECHO-Chem}\\
        \midrule
        \quad \texttt{ECHO-Charge} & 170,367 & 72.49$_{\pm12.48}$ & 2.09$_{\pm0.04}$ & 151.32$_{\pm25.16}$ & 23.54$_{\pm2.54}$ & 2 & 2 & 1  \\
        \quad \texttt{ECHO-Energy} & 196,528 & 73.73$_{\pm13.22}$ & 2.09$_{\pm0.04}$ & 153.84$_{\pm26.58}$ & 23.61$_{\pm2.59}$ & 2 & 2 & 1  \\
        \bottomrule
    \end{tabular}
    \label{tab:dataset_statistics}
\end{table}

\begin{table}[t]
\caption{Overview of the used LRIM Graph Benchmark taken from \citep{mathys_lrim_2025}}\label{tab:dataset_specs}
\centering
\resizebox{0.80\linewidth}{!}{
\begin{tabular}{lrrrrrrrr}
\toprule
Dataset & \makecell{$\sigma$ \\ easy} & \makecell{$\sigma$ \\ hard} & \makecell{Graphs} & Nodes & Edges & \makecell{Avg. Eff. \\ Resistance} & \makecell{Avg. \\ Short. Path} & Diameter \\
\midrule
LRIM-16   & 1.5 & 0.6 & 10,000 &   256   &      512   & 0.49 &   8.03  &   16  \\
\bottomrule
\end{tabular}
}
\end{table}

\begin{table}[t]
    \caption{Statistics of the Peptides LRGB datasets taken from \citep{dwivedi_long_2022}.
    }\label{tab:data_summary}
    \centering
    \setlength\tabcolsep{5pt}
    \scalebox{0.82}{
    \begin{tabular}{l r rrr rr cc}
    \toprule
    \multirow{2}{*}{\textbf{Dataset}} & \multirow{2}{*}{\shortstack[c]{\textbf{Total} \\ \textbf{Graphs}}} & \multirow{2}{*}{\shortstack[c]{\textbf{Total} \\ \textbf{Nodes}}} & \multirow{2}{*}{\shortstack[c]{\textbf{Avg} \\ \textbf{Nodes}}} &
    \multirow{2}{*}{\shortstack[c]{\textbf{Mean} \\\textbf{Deg.}}} &
    \multirow{2}{*}{\shortstack[c]{\textbf{Total} \\ \textbf{Edges}}} &
    \multirow{2}{*}{\shortstack[c]{\textbf{Avg} \\ \textbf{Edges}}} &
    \multirow{2}{*}{\shortstack[c]{\textbf{Avg} \\ \textbf{Short.Path.}}} &
    \multirow{2}{*}{\shortstack[c]{\textbf{Avg} \\ \textbf{Diameter}}}\\
    & & & & & & & & \\\midrule
    Peptides Func& 15,535 & 2,344,859 & 150.94 & 2.04 & 4,773,974 & 307.30 & 20.89$\pm$9.79 & 56.99$\pm$28.72\\
    Peptides Struct & 15,535 & 2,344,859 & 150.94 & 2.04 & 4,773,974 & 307.30 & 20.89$\pm$9.79 & 56.99$\pm$28.72\\
    \bottomrule
    \end{tabular}
    }
\end{table}

\section{Experimental Details}\label{app:experimental details}

All models are implemented using PyTorch and PyTorch Geometric (PyG). We use the official Mamba2 \citep{dao_transformers_2024} implementation for the SSM backbone. We will release the codebase upon acceptance of the paper.

We report the key architectural hyperparameters in Table \ref{tab:gsm_config}. The Feed-Forward Networks (FFN) throughout the model consist of two-layer MLPs with hidden dimension 4× the input dimension and GeLU activation. We apply LayerNorm after each component (SSM, FFN, Graph Mixing) following standard practice. For graph-level tasks, the graph decoder pools node representations using max, sum, and mean pooling, concatenates these representations, and projects them linearly to the hidden dimension before applying the decoder MLP. For node-level tasks, the decoder operates directly on node representations.

\begin{table}[ht]
\centering
\caption{Model component ablation for LGSM on the ECHO-Synth Datasets. We report both logMSE, which is optimized as a loss during training as well as the mean absolute error. The ablated components are either part of the sequence design (length, normalization) or of the block design. There we also ablate either using a simplified SSM with learnable matrices A,B,C or not using the ssm module at all. All values are reported over 3 seeds and are evaluated on the validation split.}
\label{tab:components}
\resizebox{0.8\textwidth}{!}{
\begin{tabular}{lrrrrrr}
\toprule
\multicolumn{1}{c}{Variant} & \multicolumn{2}{c}{DIAM} & \multicolumn{2}{c}{SSSP} & \multicolumn{2}{c}{ECC} \\
\cmidrule(lr){2-3}
\cmidrule(lr){4-5}
\cmidrule(lr){6-7}
 & logMSE $\downarrow$ & MAE $\downarrow$ & logMSE $\downarrow$ & MAE $\downarrow$ & logMSE $\downarrow$ & MAE $\downarrow$ \\
\midrule
baseline & 0.450 \tiny{$\pm$ 0.025} & 0.829 \tiny{$\pm$ 0.033} & {-1.405 \tiny{$\pm$ 0.184}} & 0.077 \tiny{$\pm$ 0.013} & {1.122 \tiny{$\pm$ 0.008}} & {2.381 \tiny{$\pm$ 0.053}} \\
\midrule
seq length 1 & 0.943 \tiny{$\pm$ 0.006} & 1.839 \tiny{$\pm$ 0.065} & 1.738 \tiny{$\pm$ 0.000} & 5.970 \tiny{$\pm$ 0.001} & 1.640 \tiny{$\pm$ 0.000} & 5.497 \tiny{$\pm$ 0.000} \\
seq length 5 & 0.538 \tiny{$\pm$ 0.022} & 0.855 \tiny{$\pm$ 0.028} & 1.545 \tiny{$\pm$ 0.000} & 4.338 \tiny{$\pm$ 0.015} & 1.573 \tiny{$\pm$ 0.002} & 5.068 \tiny{$\pm$ 0.011} \\
seq length 10 & 0.567 \tiny{$\pm$ 0.076} & 0.916 \tiny{$\pm$ 0.092} & 1.231 \tiny{$\pm$ 0.001} & 2.494 \tiny{$\pm$ 0.019} & 1.538 \tiny{$\pm$ 0.002} & 4.848 \tiny{$\pm$ 0.012} \\
seq length 20 & 0.462 \tiny{$\pm$ 0.011} & 0.867 \tiny{$\pm$ 0.023} & 0.380 \tiny{$\pm$ 0.006} & 0.601 \tiny{$\pm$ 0.020} & 1.386 \tiny{$\pm$ 0.018} & 3.807 \tiny{$\pm$ 0.065} \\
seq normalization none & 0.549 \tiny{$\pm$ 0.041} & 0.911 \tiny{$\pm$ 0.069} & -- & -- & 1.201 \tiny{$\pm$ 0.013} & 2.613 \tiny{$\pm$ 0.023} \\
seq normalization row & -- & -- & -0.538 \tiny{$\pm$ 0.073} & 0.167 \tiny{$\pm$ 0.031} & -- & -- \\
\midrule
num blocks 1 & 0.430 \tiny{$\pm$ 0.028} & 0.830 \tiny{$\pm$ 0.029} & -0.119 \tiny{$\pm$ 0.094} & 0.457 \tiny{$\pm$ 0.064} & 1.191 \tiny{$\pm$ 0.016} & 2.721 \tiny{$\pm$ 0.071} \\
num blocks 2 & 0.429 \tiny{$\pm$ 0.052} & {0.816 \tiny{$\pm$ 0.072}} & -0.713 \tiny{$\pm$ 0.161} & 0.181 \tiny{$\pm$ 0.030} & 1.155 \tiny{$\pm$ 0.009} & 2.507 \tiny{$\pm$ 0.028} \\
use graph mixing False & 0.521 \tiny{$\pm$ 0.007} & 0.928 \tiny{$\pm$ 0.048} & -1.224 \tiny{$\pm$ 0.153} & 0.079 \tiny{$\pm$ 0.009} & 1.128 \tiny{$\pm$ 0.002} & 2.423 \tiny{$\pm$ 0.025} \\
use simple ssm True & 0.551 \tiny{$\pm$ 0.105} & 1.023 \tiny{$\pm$ 0.144} & -1.280 \tiny{$\pm$ 0.008} & {0.075 \tiny{$\pm$ 0.012}} & 1.627 \tiny{$\pm$ 0.000} & 5.413 \tiny{$\pm$ 0.003} \\
\bottomrule
\end{tabular}
}
\end{table}

We use the AdamW optimizer with no weight decay, and gradient clipping with maximum norm 1.0. All models are trained to minimize log MSE. Following the ECHO benchmark convention, we normalize the labels during training for ECHO-Synth tasks and use original labels for ECHO-Chem tasks. For ECHO-Energy predictions, we report $10^{\text{prediction}}$ as specified in the ECHO benchmark repository. 

\paragraph{Baseline Comparisons.}
We take the baselines reported by the ECHO benchmark ~\citep{miglior_can_2025}: namely A-DGN ~\citep{gravina_anti-symmetric_2022}, DREW ~\citep{gutteridge_drew_2023}, GCN ~\citep{kipf_semi-supervised_2017}, GCNII ~\citep{chen_simple_2020}, GIN ~\citep{xu_how_2018}, GPS ~\citep{rampasek_recipe_2022}, GRIT ~\citep{ma_graph_2023}, GraphCON ~\citep{rusch_graph-coupled_2022}, PH-DGN ~\citep{heilig_port-hamiltonian_2024}, SWAN ~\citep{gravina_oversquashing_2025}.

\begin{table}[ht]
\centering
\caption{Best LGSM hyperparameter configurations per task on the ECHO benchmark, selected by validation logMSE.}
\label{tab:gsm_config}
\resizebox{\linewidth}{!}{
\begin{tabular}{llrlrrrrr}
\toprule
Task & Seq. Type & Length & Seq. Normalization & Hidden Dim & Blocks & Batch Size & Learning Rate $ \times 10^{-4}$ & Max Epochs \\
\midrule
ECHO-Charge & NBT & 40 & row & 64 & 4 & 32 & 3 & 200 \\
ECHO-Diam & NBT & 40 & row & 64 & 4 & 32 & 3 & 200 \\
ECHO-Ecc & NBT & 40 & row & 64 & 4 & 32 & 3 & 200 \\
ECHO-Energy & NBT & 32 & none & 64 & 2 & 32 & 3 & 200 \\
ECHO-Sssp & NBT & 40 & none & 64 & 4 & 32 & 3 & 200 \\
\bottomrule
\end{tabular}
}
\end{table}

\subsection{LGSM Component Ablation}

We ablate the impact of key design choices of the LGSM blocks, baseline refers to the best performing methods as reported in Table \ref{tab:gsm_config}. Recall, that models are trained to minimize logMSE (and MAE is an eval metric). We vary the sequence lengths, number of blocks, graph mixing amongst other components. Further, we ablate replacing the Mamba architectures with as simplified SSM, which consists of learnable matrices A,B,C. The results in Table \ref{tab:components} indicate that both the sequence length and processing depth, but also the graph mixing and role of the SSM are important for the final performances on the ECHO Synth Datasets.

\subsection{Evaluation on LRIM Graph Benchmark}

In order to further evaluate the capabilities of our LGSM method, we evaluate on the LRIM Graph Benchmark \cite{mathys_lrim_2025}, a benchmark specifically tailored to the precise assessment of long-range capabilities. We evaluate and report the performance in Table \ref{tab:lrim_gsm}, where LGSM performs better than the message-passing baselines and comparable to the computationally more expensive transformer based approaches. Further, we ablate the impact of the information and processing depth, which are central parts of the LGSM design. Figure \ref{fig:lrim} illustrates that both dimensions contribute towards better performance, however, solely increasing processing depth without also increasing the information depth yields diminishing returns.

\subsection{Evaluation on Peptides Datasets}

\begin{table}[h!]

    \centering
    \caption{Results for Peptides-func and Peptides-struct averaged over 3 training seeds. Baseline results are taken from \cite{dwivedi_long_2022}, \cite{gutteridge_drew_2023} and \citet{eliasof_graph_2025}. Note that all MPNN-based methods include structural and positional encoding, whereas ours does not.
    }
    \label{tab:results_lrgb_complete}
    
    \scriptsize
    \begin{tabular}{@{}lcc@{}}
    \hline\toprule
    \multirow{2}{*}{\textbf{Model}} & \textbf{Peptides-func}  & \textbf{Peptides-struct}              
    \\
    & \scriptsize{AP $\uparrow$} & \scriptsize{MAE $\downarrow$}
    \\ \midrule  
    
    \midrule
    \textbf{Multi-hop GNNs}\\
    $\,$ DIGL+MPNN           & 64.69$_{\pm00.19}$         & 0.3173$_{\pm0.0007}$ \\
    $\,$ DIGL+MPNN+LapPE     & 68.30$_{\pm00.26}$         & 0.2616$_{\pm0.0018}$ \\
    $\,$ MixHop-GCN          & 65.92$_{\pm00.36}$         & 0.2921$_{\pm0.0023}$ \\
    $\,$ MixHop-GCN+LapPE    & 68.43$_{\pm00.49}$         & 0.2614$_{\pm0.0023}$ \\
    $\,$ DRew-GCN            & 69.96$_{\pm00.76}$         & 0.2781$_{\pm0.0028}$ \\
    $\,$ DRew-GCN+LapPE             & {71.50$_{\pm0.44}$}   & 0.2536$_{\pm0.0015}$ \\
    $\,$ DRew-GIN            & 69.40$_{\pm0.74}$         & 0.2799$_{\pm0.0016}$ \\
    $\,$ DRew-GIN+LapPE      & {71.26$_{\pm0.45}$}   & 0.2606$_{\pm0.0014}$ \\
    $\,$ DRew-GatedGCN       & 67.33$_{\pm0.94}$         & 0.2699$_{\pm0.0018}$ \\
    $\,$ DRew-GatedGCN+LapPE & 69.77$_{\pm0.26}$         & 0.2539$_{\pm0.0007}$ \\
    
    \midrule
    
    \textbf{Transformers} \\
    $\,$ Transformer+LapPE & 63.26$_{\pm1.26}$ & 0.2529$_{\pm0.0016}$           \\
    $\,$ SAN+LapPE         & 63.84$_{\pm1.21}$ & 0.2683$_{\pm0.0043}$           \\
    $\,$ GraphGPS+LapPE    & 65.35$_{\pm0.41}$ & 0.2500$_{\pm0.0005}$   \\
    \midrule
    \textbf{Modified and Re-evaluated}$^\ddag$\\
    $\,$ GCN
    & 68.60$_{\pm0.50}$ & 0.2460$_{\pm0.0007}$\\
    $\,$ GINE
    & 66.21$_{\pm0.67}$ & 0.2473$_{\pm0.0017}$\\
    $\,$ GatedGCN
    & 67.65$_{\pm0.47}$ & 0.2477$_{\pm0.0009}$\\
    $\,$ DRew-GCN+LapPE
    & 69.45$_{\pm0.21}$        & 0.2517$_{\pm0.0011}$\\
    $\,$ GraphGPS+LapPE
    & 65.34$_{\pm0.91}$ & 0.2509$_{\pm0.0014}$\\
    $\,$ GCN+ & 72.61 $_{\pm0.11}$& 0.2421 $_{\pm0.0016}$\\
    $\,$ GIN+ & 70.59 $_{\pm0.89}$& 0.2429 $_{\pm0.0019}$\\
    $\,$ GatedGCN+ & 70.06 $_{\pm0.33}$& 0.2431 $_{\pm0.0020}$\\

    \midrule
    \textbf{GNNs} \\
    $\,$ GRAND      & 57.89$_{\pm0.62}$ & 0.3418$_{\pm0.0015}$ \\
    $\,$ GraphCON   & 60.22$_{\pm0.68}$ & 0.2778$_{\pm0.0018}$ \\
    $\,$ A-DGN      & 59.75$_{\pm0.44}$ & 0.2874$_{\pm0.0021}$ \\
    $\,$ {SWAN} & 67.51$_{\pm0.39}$ & 0.2485$_{\pm0.0009}$\\
    $\,$ Graph-Mamba & 67.39$_{\pm0.87}$ &  {0.2478$_{\pm0.0016}$} \\
    $\,$ Neural Walker & 70.96$_{\pm0.78}$ &  {0.2463$_{\pm0.0005}$}\\
    $\,$ GMN & 70.71$_{\pm0.83}$ &  {0.2473$_{\pm0.0025}$} \\
    $\,$ GRAMA$_{\textsc{GatedGCN}}$ & 70.49$_{\pm0.51}$  & {0.2459$_{\pm 0.0020}$}  \\
    \midrule
    \textbf{Ours} \\
    $\,$ LGSM (ours)   & 66.85$_{\pm 1.36}$ &  0.2470 $_{\pm 0.0019}$ 	\\
    \bottomrule\hline
    \end{tabular}

    \end{table}

We follow the setup from \citet{tonshoff_where_2024}: AdamW with cosine scheduling for 250 epochs, decoder head, batch size 200 and ran a small selection of hyperparameters (seq. type, length, num\_blocks). 
We did not observe a similar improvement with longer sequences. Instead, on Peptides it seemed to matter most to have choices which implicitly regularize more (shorter sequences, adj) in order to incentivize good generalization. Moreover, we observed consistent (and significant) gaps between train and validation performances which might suggest that regularization is more important for this task compared to the impact of the architecture.

\begin{table}[ht]
\centering
\caption{Evaluation of LGSM on the LRIM Graph Benchmark. We report both asymptotic runtime which is dependent on if the input sequence is preprocessed or not as well as the performance metric in logMSE. The LGSM uses 8 blocks with sequence length 32 and surpasses the message-passing based baselines while being very close in performance to the computationally more expensive transformer based approaches.}
\label{tab:lrim_gsm}
\resizebox{0.8\textwidth}{!}{
\begin{tabular}{lllr}
\toprule
 & Preprocessing & Computation & LRIM-16-hard $\downarrow$\\
\midrule
GIN & - & $\mathcal{O}(L\cdot E)$ & {-2.406 $\pm$ \tiny{0.148}}\\
GatedGCN & - & $\mathcal{O}(L\cdot E)$ & {-3.919 $\pm$ \tiny{0.223}}\\
\midrule
GatedGCN-VN$_G$ & $\mathcal{O}(N)$ & $\mathcal{O}(L\cdot E+L\cdot N)$ & {-3.756 $\pm$ \tiny{0.063}}\\
\midrule
GPS-Base & - & $\mathcal{O}(L\cdot N^2)$ & {-4.340 $\pm$ \tiny{0.101}}\\
GPS-RWSE & \RWSE & $\mathcal{O}(L\cdot N^2)$ & {-4.345 $\pm$ \tiny{0.065}}\\
GPS-LapPE & \LapPE & $\mathcal{O}(L\cdot N^2)$ & {-4.248 $\pm$ \tiny{0.110}}\\
\midrule
LGSM & $\mathcal{O}(L\cdot E)^{\dagger}$ \text{ or }$\mathcal{O}(1)^{\ddagger}$  & $\mathcal{O}(D\cdot L\cdot E)$ & {-4.284 $\pm$ \tiny{0.133}}\\
\bottomrule
\end{tabular}
}
\end{table}

\begin{figure}[t]
    \centering
    \includegraphics[width=0.8\linewidth]{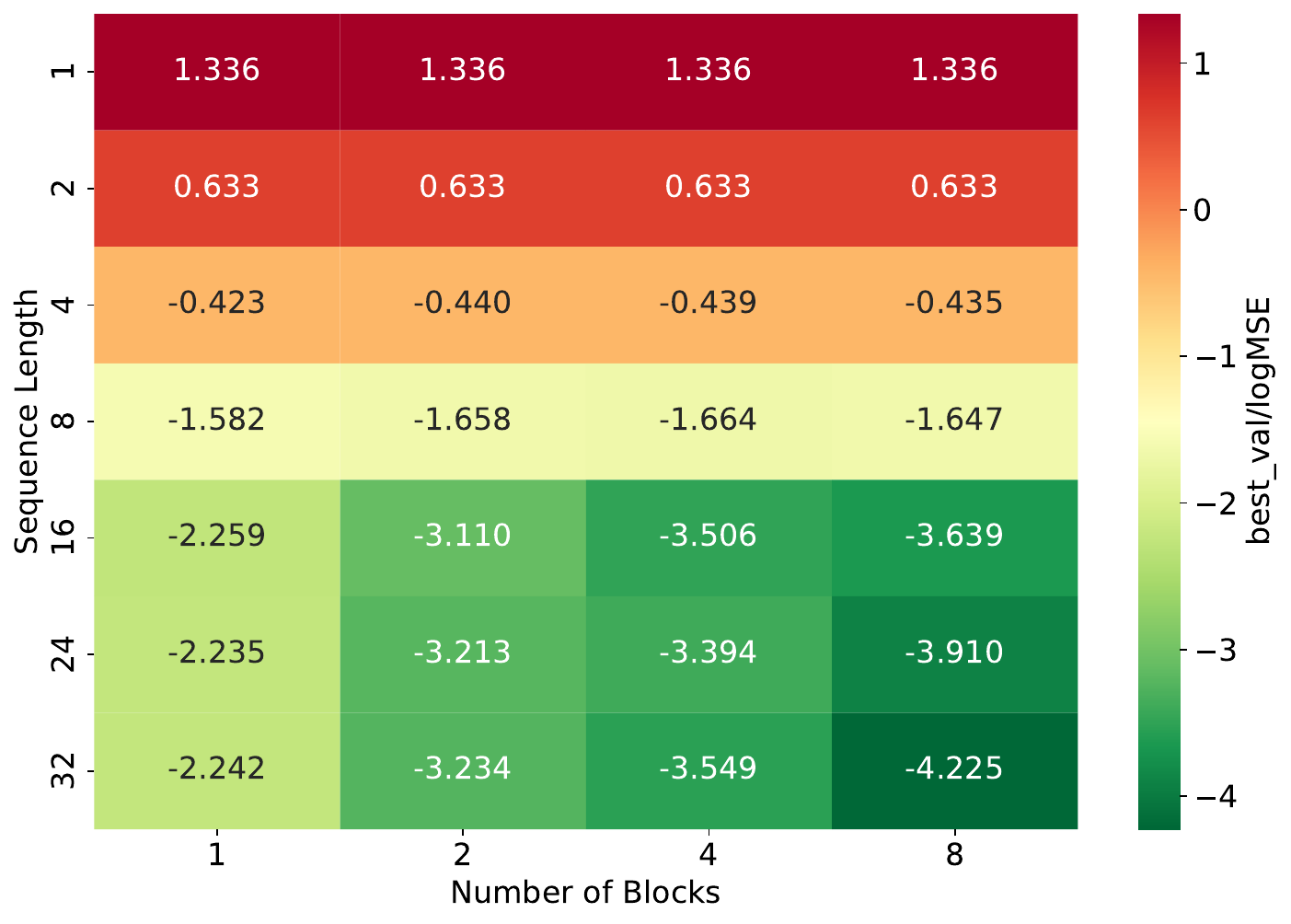}
    \caption{Ablation of the impact of scaling either information depth or processing depth with LGSM evaluated on the LRIM Graph Benchmark. We see that increasing either dimension has a positive effect, while solely increasing processing depth without incorporating additional information yields diminishing returns.}
    \label{fig:lrim}
\end{figure}

\section{Empirical Runtime Measurement}

We quantitatively evaluate our method on sparse ER graph (connectivity log n /n) of size 256 for various sequence lengths (layers respectively). Moreover, we provide results for fixed sequence/layers=16 for variously sized sparse ER graphs. We also consider a simple GCN and GPS baseline based on the pyg modules with FFNs. We also provide LGSM baselines with 1,2 blocks as well as the simplified SSM variant (learned A,B,C matrices instead of Mamba). Note that as mentioned in our limitations the empirical efficiency was not our main focus. As such, we recompute the sequences each forward pass and compute row-wise (which requires more memory than columnwise). Therefore, memory grows proportional to the sequence length and graph size. Measurements are taken over 200 graphs and measured on an RTX 3090 with 24 GB VRAM.
In Figure \ref{fig:timing_seq} we measure timing and memory requirements for a fixed graph of size 256 for multiple sequence lengths. Similarly, in Figure \ref{fig:timing_graph} we report time and memory measurements for a fixed sequence length (or number of layers) of 16 across graphs up to size 8192. 

\begin{figure}
    \centering
    \includegraphics[width=0.9\linewidth]{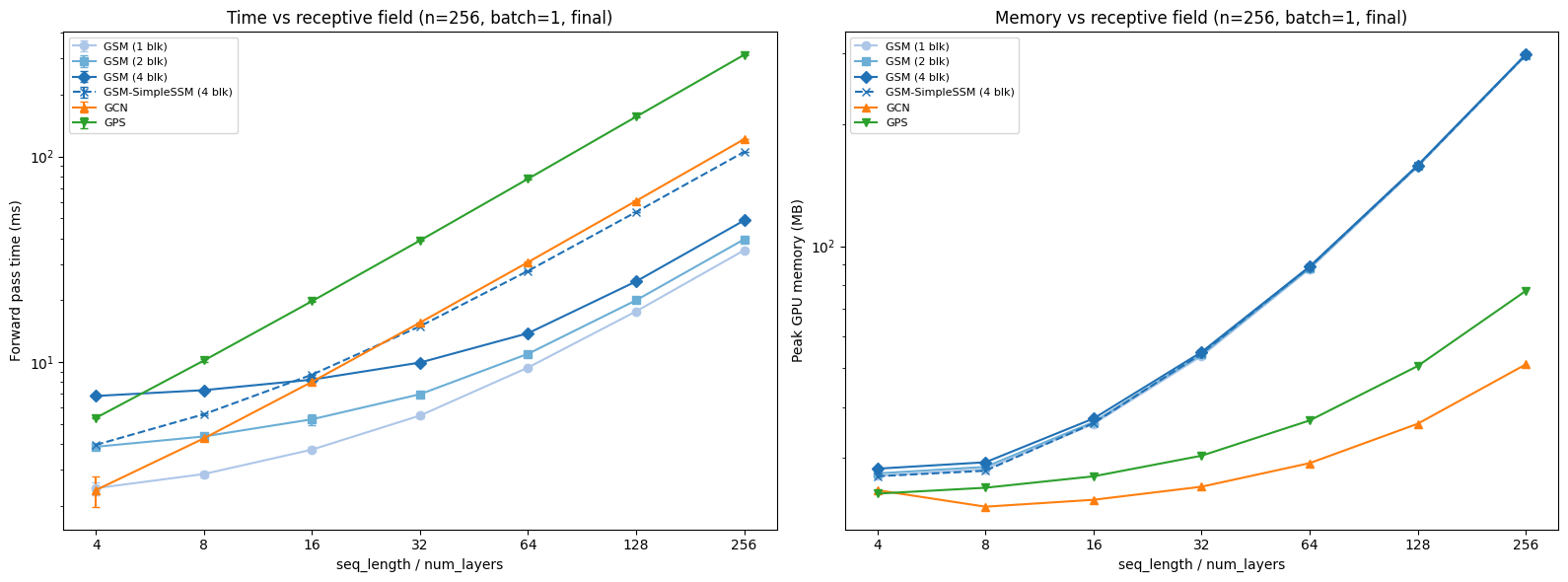}
    \caption{Measurements for varying number of layers on ER graphs of size 256. LGSM memory grows roughly linearly with $L$. Memory use is clearly driven by sequence length rather than number of blocks. The timing of LGSM stays almost constant over a range of small to medium graphs. For large graphs, the forward pass time of GPS blows up the most due to global attention, while LGSM stays between GCN and GPS.}
    \label{fig:timing_seq}
\end{figure}

\begin{figure}
    \centering
    \includegraphics[width=0.9\linewidth]{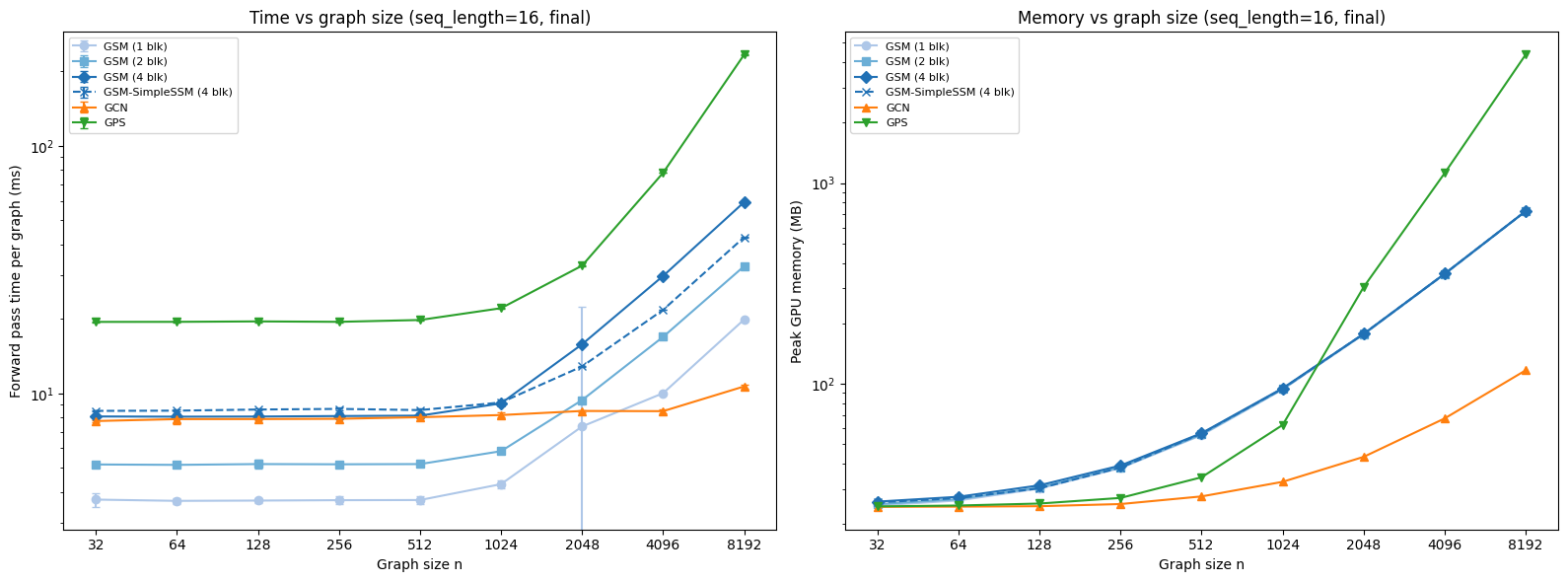}
    \caption{Measurements for fixed number of layers (16) across multiple graph sizes. The timing of LGSM stays almost constant over a range of small to medium graphs. For large graphs, the forward pass time of GPS blows up the most due to global attention, while LGSM stays between GCN and GPS. LGSM memory use increases steadily with graph size, scaling linearly and better than GPS for large graphs.}
    \label{fig:timing_graph}
\end{figure}

\clearpage
\section{Sequence Extraction}
\label{sec:seq_extraction}

We consider two different sequence extraction mechanisms. Let $\An = \mathbf D_\G^{-\frac{1}{2}} \mathbf A_\G \mathbf D_\G^{-\frac{1}{2}}$ be the normalized adjacency matrix, and let $\mathbf B_\G$ be the non-backtracking matrix as defined in Definition~\ref{def:NBT}. Depending on the choice of matrix, the $k$th sequence element is given by $\smash{\mathbf S^{(k)} = \An^k X}$ and $\smash{\mathbf S^{(k)} = \mathbf B_\G^{(k)} X}$, respectively.

We analyze the influence of different nodes in the sequence extraction. The influence of node $w$ on the $k$th sequence element of $v$ is defined as
\begin{equation}
    I_{v,k}(w) = \mathbf e^T \left[\frac{\partial \mathbf s^{(k)}_v}{\partial \mathbf x_w}\right] \mathbf e \Big / \left(\sum_{u \in V} \mathbf e^T \left[\frac{\partial \mathbf s^{(k)}_v}{\partial \mathbf x_u}\right] \mathbf e\right)
\end{equation}
where e is the all ones vector. Note that this definition is similar to the definition of influence distributions in ~\citep{xu_representation_2018}, which analyzes influence distribution in a GNN, whereas we focus on the sequence extraction part.

\begin{proposition}
    Let $\mathbf Z^{(k)}$ be either $\mathbf B_\G^{(k)}$ or $\An^{k}$. The influence of node $w$ on the $k$th sequence element of $v$ is given by
    $I_{v,k}(w) = \mathbf Z_{vw}^{(k)} / \sum_{u \in V} \mathbf Z_{vu}^{(k)}$.
\end{proposition}
\begin{proof}
    Since feature dimensions are independent, we have
    $$
        \left[\frac{\partial (\mathbf s^{(k)}_v)_i}{\partial (\mathbf x_w)_j}\right] 
        = \begin{cases} \mathbf Z_{vw}^{(k)} & i=j \\ 0 & i \neq j\end{cases}
    $$
    Hence $\mathbf e^T \left[\frac{\partial \mathbf s^{(k)}_v}{\partial \mathbf x_w}\right] \mathbf e = d\cdot \mathbf Z_{vw}^{(k)}$. After normalizing, $I_{v,k}(w) = \mathbf Z_{vw}^{(k)} / \sum_{u \in V} \mathbf Z_{vu}^{(k)}$
\end{proof}

We show that there can be an exponential difference between the influence of a node $w$ on $v$, depending on which matrix is used for sequence extraction. 

\influence*

\begin{lemma}
    \label{lem:influence_A_nbt}
    Suppose $v,w$ satisfy $\mathrm{dist}(v,w)=k$ and the radius-$k$ neighborhood of $v$ is a $d$-regular tree. Let $I^A_{v,k}(w)$ (resp.\ $\smash{I^B_{v,k}(w)}$) denote the influence of $w$ on the $k$th sequence element at $v$ induced by adjacency extraction $\mathbf S^{(k)}=\An^kX$ (resp.\ non-backtracking extraction $\smash{\mathbf S^{(k)}=\mathbf B^{(k)}_\G X}$). Then
    \[
    \frac{I^A_{v,k}(w)}{I^B_{v,k}(w)}=\left(\frac{d-1}{d}\right)^{k-1},
    \]
    and in particular $I^A_{v,k}(w)$ is exponentially smaller than $I^B_{v,k}(w)$ in $k$. 
\end{lemma}
\begin{proof}
    Since $\text{dist}(v,w)=k$, there is exactly one path from $v$ to $w$, which means $(\mathbf A_\G^k)_{vw} = 1 = (\mathbf B_\G^{(k)})_{vw}$. We have $\smash{(\An^k)_{vw} = d^{-k}}$, since for $d$-regular graphs, $\smash{\An^k = d^{-k} \mathbf A_\G^k}$.  

    To compute the influence distributions, we need to compute the denominators. In a $d$-regular graph, the number of different walks of length $k$ from $v$ is $d^k$, which is also given by $\sum_{u \in V} (\mathbf A_\G^{k})_{vu}$. Hence, $\sum_{u \in V} (\An^{k})_{vu} = d^k/d^k = 1$. 
    
    For the non-backtracking matrix, normalization is done by $\sum_{u \in V} (\mathbf B_\G^{(k)})_{vu}=d\cdot (d-1)^{k-1}$, since a walk from node $v$ can first use any of $d$ edges and then only has $d-1$ options for the next steps. 

    Putting these together, we have 
    $$
    \frac{I^A_{v,k}(w)}{I^B_{v,k}(w)} = \frac{d^{-k}}{1/d(d-1)^{k-1}} = \left(\frac{d-1}{d}\right)^{k-1}
    $$
\end{proof}

\end{document}